\documentclass[a4paper,11pt,twoside]{article}
\pdfoutput=1

\usepackage{epsfig}

\usepackage{graphicx}
\usepackage{subfigure}
\usepackage{amsmath}
\usepackage{amssymb}
\usepackage{hujititlepage}
\usepackage{amsthm}
\usepackage[font={small,it}]{caption}
\usepackage{wrapfig}
\usepackage{indentfirst}

\renewcommand{\i}{{\mathrm{i}}}

\usepackage{graphicx,subfigure}
\usepackage{epsfig}
\usepackage{amsmath}
\usepackage{amssymb}
\usepackage{algorithm}
\usepackage{algorithmic}
\usepackage{url}
\usepackage{enumerate}
\usepackage{amsfonts}
\usepackage{boxedminipage}
\usepackage{xcolor}
 \usepackage{framed}  
\usepackage{rotating}
\usepackage{array}
\usepackage{multirow}
\usepackage{color}
\usepackage{tikz}
\usepackage{pgfplots}
\usepackage{mathabx}
\usepackage{tabularx,ragged2e,booktabs,caption}
\usepackage{thmtools}

{
\begin{center}
\begin{boxedminipage}{0.8\linewidth}
\begin{center}
\textbf{\texttt{#1}}
\end{center}
\rm
\begin{tabbing}
....\=...\=...\=...\=...\=  \+ \kill
} %
{\end{tabbing} 
\end{boxedminipage} \end{center} 
}

{
\vspace{0.01cm}
\begin{center}
\begin{boxedminipage}{0.8\linewidth}
\begin{center}
\textbf{\texttt{#1}}
\end{center}
} %
{
\end{boxedminipage} \end{center} \vspace{0.3cm}
}

\newtheorem{definition}{Definition}
\newtheorem*{definition*}{Definition}

\newtheorem{corollary}{Corollary}
\newtheorem*{corollary*}{Corollary}
\newtheorem{theorem}{Theorem}
\newtheorem*{theorem*}{Theorem}

\newtheorem*{claim*}{Claim}

\newcommand{\reals}{\mathbb{R}}

\newcommand{\bC}{\mathbb{C}}

\newcommand{\bA}{\mathbb{A}}

\newcommand{\cD}{\mathcal{D}}

\newcommand{\cH}{\mathcal{H}}

\newcommand{\abs}[1]{\left|#1\right|}
\newcommand{\norm}[1]{\left \Vert#1\right\Vert}
\newcommand{\brac}[1]{\left[{#1}\right]}

\makeatletter
\def\moverlay{\mathpalette\mov@rlay}
\def\mov@rlay#1#2{\leavevmode\vtop{%
   \baselineskip\z@skip \lineskiplimit-\maxdimen
   \ialign{\hfil$\m@th#1##$\hfil\cr#2\crcr}}}
\newcommand{\charfusion}[3][\mathord]{
    #1{\ifx#1\mathop\vphantom{#2}\fi
        \mathpalette\mov@rlay{#2\cr#3}
      }
    \ifx#1\mathop\expandafter\displaylimits\fi}
\makeatother

\newcommand{\diff}{\mathrm{d}}

\newcommand{\pderiv}[2]{\frac{\partial #1}{\partial #2}}

\DeclareMathOperator*{\argmax}{arg\,max}

\DeclareMathOperator{\relu}{ReLU}
\DeclareMathOperator{\softmax}{softmax}

\renewcommand{\eqref}[1]{Equation~(\ref{#1})}

\newcommand{\parenth}[1]{\left({#1}\right)}

\newcommand{\handout}[5]{
   \renewcommand{\thepage}{#1-\arabic{page}}
   \noindent
   \begin{center}
   \framebox{
      \vbox{
    \hbox to 5.78in { {\bf (67577) Introduction to Machine Learning}
         \hfill #2 }
       \vspace{4mm}
       \hbox to 5.78in { {\Large \hfill #5  \hfill} }
       \vspace{2mm}
       \hbox to 5.78in { {\it #3 \hfill #4} }
      }
   }
   \end{center}
   \vspace*{4mm}
}

\newcommand{\remove}[1]{}

\begin{document}
\pagenumbering{gobble}

\title{On Complex Valued Convolutional Neural Networks}
\author{Nitzan Guberman}
\maketitle

\begin{center}
\section*{Abstract}
\end{center}

Convolutional neural networks (CNNs) are the cutting edge model for supervised machine learning in computer vision. In recent years CNNs have outperformed traditional approaches in many computer vision tasks such as object detection, image classification and face recognition. CNNs are vulnerable to overfitting, and a lot of research  focuses on finding regularization methods to overcome it. One approach is designing task specific models based on prior knowledge.

Several works have shown that properties of natural images can be easily captured using complex numbers. Motivated by these works, we present a variation of the CNN model with complex valued input and weights. We construct the complex model as a generalization of the real model. Lack of order over the complex field raises several difficulties both in the definition and in the training of the network. We address these issues and suggest possible solutions.

The resulting model is shown to be a restricted form of a real valued CNN with twice the parameters. It is sensitive to phase structure, and we suggest it serves as a regularized model for problems where such structure is important. This suggestion is verified empirically by comparing the performance of a complex and a real network in the problem of cell detection. The two networks achieve comparable results, and although the complex model is hard to train, it is significantly less vulnerable to overfitting. We also demonstrate that the complex network detects meaningful phase structure in the data.

\cleardoublepage
\section*{Acknowledgments}

I would like to thank my supervisor, Prof. Amnon Shashua, who had introduced me to the exciting field of computer vision, and guided me throughout this research. I would also like to express my gratitude for my peers, Nadav Cohen, Or Sharir, Ronen Tamari, Erez Peterfreund, Nomi Vinokurov, Tamar Elazari , Inbar Huberman and Roni Feldman (please forgive me if I forgot someone). They have made this experience much more meaningful, and enjoyable.

Finally, and most importantly, I am greatly thankful to my family for their support. Most especially, I thank my husband Yahel. My gratitude for his help and support is beyond words.

\clearpage

\tableofcontents
\cleardoublepage
\pagenumbering{arabic}

\section{Introduction} \label{intro}

Learning algorithms have had a huge impact on numerous fields in computer science, and found many applications in diverse fields such as computer vision, bioinformatics, robot locomotion and speech recognition. These algorithms avoid hand crafting solutions to specific problems by opting instead to "learn" and adapt according to a set of examples called the training set. A learning algorithm consists of a rough model and a method of tuning its parameters to fit the training set.

Neural networks are an example for such a model. Inspired by the human brain, they are composed of many interconnected simple computational units, whose combination results in an elaborate function. This model was first introduced in the 1940's in \cite{Hebb1949}, and has been studied intermittently in the following years. A major breakthrough occurred in the 1990's, for example in  \cite{Waibel1989,lecun1998gradient,LecunBengioHandbook}, with the advent of convolutional neural networks (CNNs), a restricted form of neural networks specifically adapted to natural images . However, it was not until the past decade that an increase in computational and data resources enabled successful learning with CNNs.

CNNs have been a game changer in computer vision, significantly outperforming state of the art results for many tasks. Examples include image classification \cite{Krizhevsky2012}, object detection \cite{Girshick2014}, and face recognition \cite{Taigman2014}. In the latter, human level performance was reached. In the past years, much of the research in computer vision was focused on utilizing CNNs for new problems, and improving the existing CNN model and its training process. 

One avenue of ongoing effort, is in developing methods to overcome overfitting. Overfitting is the learning algorithm's habit of fitting the training set "too well", at the expense of unseen examples. It is a major challenge with expressive models such as CNNs. One approach for restraining overfitting is by restricting the CNN model based on prior knowledge.

In this work, we suggest a variation of the CNN model, with complex valued input and parameters. Complex numbers have long proved useful for handling images (e.g. the Fourier transform is complex valued), and have been considered in a neural network related context. For example, synchronization effects exist in the human brain, and are suspected to play a key role in the visual system. Such effects are lacking in mainstream neural network implementations. In  \cite{RaoRavishankar2008,Reichert2013}, synchronization was introduced to neural networks via complex numbers, and was used for segmenting images into separate objects. Another notable example for the use of complex numbers in networks is presented in \cite{Bruna2013}. In this work, robust image representations are generated using a degenerate form of a complex valued convolutional network. Using these representations the authors achieved state-of-the-art results in various tasks.

In the following, we first introduce the necessary background, and further discuss the prior work that motivated the complex variant of the CNN model. We then describe the generalization of the model to complex numbers, and address the difficulties encountered  in the construction and optimization of the network. We show, that a complex valued CNN can be seen as a restricted form of a larger real valued CNN, and as such it has the potential of mitigating the effects of overfitting. We further characterize the complex convolution operation, and obtain that complex valued CNNs are well suited for detecting phase structure.

To test the complex network's susceptibility to overfitting, we empirically compare the complex model with an equivalent real one, in a simple problem of cell detection. We show that the networks' performance is similar, but that the complex network has a problematic optimization process. The complex network is seen to be much more resilient to overfitting, and we show that it utilizes phase structure in a similar manner to the prior work presented.

\newpage

\section{Backgroung}
In this chapter the needed background for discussing complex CNNs is laid out. The general supervised learning method is described in section \ref{intro:supervised}. Neural networks, and specifically convolutional neural networks are introduced in section \ref{intro:nn}. 

\subsection{Supervised Learning}\label{intro:supervised}
Many problems in computer vision are complicated enough to pose significant difficulties for ad-hoc algorithms. For example, constructing an algorithm to decide whether an image contains a cat or not is not straightforward. The machine learning approach avoids tailoring specific algorithms for these problems, by allowing computer programs to learn to solve such problems themselves. Supervised learning algorithms are designed to learn and adapt by observing samples of real inputs and their expected outputs.

For example, in a classification problem there are several possible labels that can be assigned to inputs. The goal is to find a classifier that assigns each input (e.g. image) the right label (e.g. "cat" or "not cat"). 
A supervised learning algorithm for this task, is given a training set of inputs and correct labels, and outputs a classifier.

More formally, let $X$ be the input space (e.g. all possible images) and $Y$ the output space (e.g. labels). Let $\cD$ be the probability distribution over $X\times Y$. An inference function describes the connection between the input and the output, \[
f:X\to Y\]

The quality of a learning algorithm is quantified by a loss function, measuring how well the inference function operates on data, not necessarily given in the training set. For every input-output pair $(x,y)$, the loss function $\ell(f(x),y))$ compares $f(x)$  with the correct output $y$, and returns a penalty. The loss of $f$ is the expected loss with respect to all possible inputs, i.e.
\[
L_{\cD}(f)=\mathbb{E}_{(x,y)\sim \cD}\left[\ell(f(x),y)\right]\]

For classification problems, a possible loss function is the $0\backslash 1$ loss, defined by \[
\ell(f(x),y))=\begin{cases}1 & f(x)\neq y \\0 & f(x)=1\end{cases}\]
In this case, $L_{\cD}(f)$ measures the percentage of successful classifications made by $f$, called its accuracy.

The learning algorithm attempts to find the function that minimizes this loss, and so can be formulated as solving the optimization problem:
\[\min_f L_{\cD}(f)\]

This optimization problem is often impossible to solve directly, so an approximated version is solved instead. For one, the "no free lunch" theorem (See chapter 5 in \cite{ShalevBook}), states that it is impossible to learn an unconstrained function. Therefore, the search is restricted to a hypothesis class $\mathcal{H}$ which is chosen according to prior knowledge about the problem. For example, a common hypothesis class is the linear functions. In addition, the probability space $\cD$ is oftentimes unknown, or too complicated to handle, so only a finite training set sampled from $\cD$ is used. The training set $S$ is comprised of pairs of inputs and outputs, drawn i.i.d from $\cD$, i.e. $ S=\{\left(x_i,y_i\right)\}_{i=1}^m \sim D^n $. The revised optimization problem, called the empirical loss minimization (ERM) rule, is given by:
\[\min_{f\in\mathcal{H}}L_S(f)\]
Where $L_S(f)$ is the empirical loss of $f$:	
\begin{equation}\label{eq:ERM}
L_S(f)=\mathbb{E}_{(x,y)\sim U(S)}\left[\ell(f(x),y)\right]=\frac{1}{\abs{S}}\sum_{i=1}^{\abs{S}}\left[\ell(f(x_i),y_i)\right]
\end{equation}

Learning with the ERM rule poses several challenges, including overfitting. A function $f$ is said to be overfitting if it fits the training set, rather than the whole input domain $\cD$. Such a function has a low empirical loss $L_S(f)$, and a high $L_{\cD}(f)$. For example, in the cat detection task, if all the cat images in the training where taken with the same background, a classifier detecting this background would be very successful for the training set, but act very poorly over general images.  

To quantify this notion, we define the approximation and estimation errors. For any $f$, the error $L_{\cD}(f)$ is composed of two parts:\[
L_{\cD}(f) = \underbrace{L_{\cD}(f_0)}_{\epsilon_\text{app}}+\underbrace{L_{\cD}(f)-L_{\cD}(f_0)}_{\epsilon_{\text{est}}}
\]
Where  the approximation error, $\epsilon_\text{app}=L_{\cD}(f_0)$, is the minimal possible error for any function from $\cH$. The estimation error, $\epsilon_\text{est}$, measures the overfitting of $f$, and stems from the fact that the algorithm uses only a sampled training set. Having a large, or expressive, hypothesis class can reduce the approximation error, but risks increasing the estimation error.

The choice of an appropriate hypothesis class is crucial to the success of the learning process, not only due to the trade-off between expressiveness and overfitting. The more expressive $\cH$ is, the larger the training set needed to achieve a low loss. Additionally, the computational complexity of the algorithm changes with the choice of $\cH$, as different classes have learning algorithms with varying complexity.

Artificial neural networks, and specifically the subclass of convolutional neural networks, have recently proven very successful for many computer vision tasks. In the following section these hypothesis classes are presented.

\subsection{Convolutional Neural Networks}\label{intro:nn}

The primary motivation behind neural networks is biological. Neural networks are inspired by the human brain, and as such are built of small computational units that communicate with each other. The combination of many such neurons and connections can execute very complex calculations.

An artificial neural network (ANN) is composed of alternating layers of two types, affine and activation function, as seen for example in figure \ref{fig:ann_example}. In an affine layer, each neuron's value is a weighted sum of the previous layer's neurons.  In an activation function layer, each neuron's value is set to be a non-linear function of exactly one neuron from the previous layer. Typical activation functions are sigmoid and $\tanh$.  

An ANN layer can be represented by a vector of its neurons' values. Given a layer $o$, a following affine layer, $z$, would be $z = W o + b$, for some matrix $W$ and vector $b$ called the layer's weights. An activation function layer would be given by $\forall i\; z_i = f(o_i)$, for some point-wise non-linear function $f$. 
 
\begin{figure}[h!]
	\centering
	\includegraphics[width=0.8\textwidth]{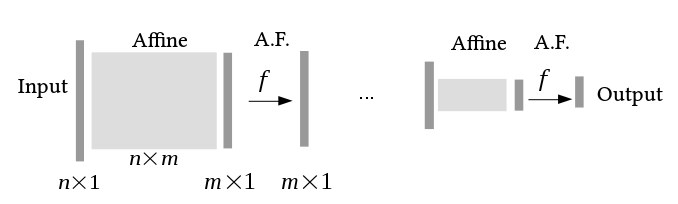}
	\caption{A typical ANN, with an $n$ dimensional input, and consecutive  affine and activation function layers. The architecture specifications include the dimensions of each layer, and the choice of activation function $f$. The weights include the matrices of each affine layer.\label{fig:ann_example}}
\end{figure}

ANNs have been around for decades before  the appearance of CNNs, a restricted form of ANNs especially designed for handling images and other natural signals. This was one of the major breakthroughs which allowed for a new level of performance in many computer vision tasks, such as image classification \cite{Krizhevsky2012}, object detection \cite{Girshick2014}, and face recognition \cite{Taigman2014}. CNNs have been reviewed extensively in the literature, cf. \cite{BengioBook,schmidhuber2015deep,lecun2015deep,LecunBengioHandbook,lecun1998gradient}, and we will present only the needed background for this work.

In CNNs, the neurons in each layer are organized as a three dimensional array rather than as a vector. The first two dimensions are called spatial, and the third is a devision to channels. The CNN model follows three principles characteristic of natural images - locality, sharing and pooling.

The locality property, is the fact that pixels depend only on their neighbors, rather than on far away pixels. Sharing is the restriction that different pixels should undergo the same processing. Demanding that an affine layer adhere to locality and sharing results in a convolution layer. In a convolution layer with input $o$, the $k^{\text{th}}$ channel is given by 
$$
o\ast K^{(k)}+b^{(k)}
$$
Where $\ast$ is the convolution operation, and $\{K^{(k)},b^{(k)}\}$ are the convolution's kernel and bias terms, respectively. The weights of the convolution layer are the kernels and bias terms of all its channels.
A general affine layer is called fully-connected in this context, to contrast it with a convolution layer.

 Pooling is used to induce invariance to small translations, which is a characteristic of natural images. A pooling layer does so by splitting each input channel into patches, and replacing each patch with a single representative value in the output layer. Typical choices the maximal or average value, in $\max$ and average pooling, respectively.

Finally,  CNNs also utilize a new activation function, the rectified linear unit ($\relu$). The $\relu$  point-wise function is given by:
$$
\relu(x)=\begin{cases} x & x\geq 0 \\ 0 & o.w. \end{cases}
$$

Traditional CNNs are composed of several repetitions of convolution, $\relu$ and pooling layers. These layers preserve the three dimensional structure of the input, while the desired output is often of a vector form. To that end, the three dimensional structure is collapsed to a vector, which serves as input to several recurrences of fully connected and $\relu$ layers. 

The architecture of the network is the configuration of its layers, and their specifications, e.g. the kernels' sizes and strides for convolution and pooling layers. An example for a CNN architecture is seen in figure \ref{fig:cnn_example}.

\begin{figure}[h]
	\centering
	\includegraphics[width=0.9\textwidth]{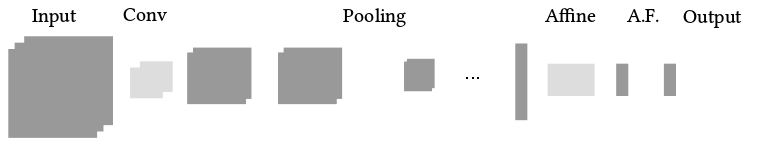}
	\caption{A typical CNN, for a three dimensional input (e.g. an RGB image). The initial layers  are convolution, $\relu$ and pooling operating over three dimensional inputs. The final layers operate over one dimensional inputs, analogous to ANNs.\label{fig:cnn_example}}
\end{figure}

The CNN architecture for a specific problem is manually chosen according to the nature of the problem, prior knowledge and trial and error.

\subsection{Optimization Methods for Training CNNs}\label{intro:opt}

Given a CNN architecture, and a labeled training set $S=\{(x_i,y_i)\}_{i=1}^m$, the learning algorithm finds the weights of the convolution and affine layers. The weights are chosen to minimize the loss function, i.e. they are the solution to the optimization problem

\[\min_{W}\sum_{i=1}^{m}\ell(f(W;x_i),y_i)\]
Where $W$ is the network's weights, $f(W;x_i)$ is the prediction given by the network with weights $W$ for input $x_i$, and $\ell$ is the loss function.

This optimization problem is typically solved using gradient based methods. These are iterative methods that use the first order approximation for the minimized function. In each step, the weights are updated by moving in the direction of the loss function's steepest descent, found by its gradient. Formally, given the weights at time $t$ - $W^{(t)}$, the weights at the next time step are:
\begin{align}\label{eq:GD}
W^{(t+1)}=w^{(t)}-\eta_t\nabla_W\parenth{\sum_{i=1}^{m}\ell(f(W^{(t)};x_i),y_i)}
\end{align}
Where $\eta_t$ is a positive scalar called the learning rate (which may be time dependent), and $\nabla_W$ is the gradient with respect to $W$. If the minimized function is not differentiable, but convex, any value from the sub gradient can replace the gradient in equation \ref{eq:GD}.

The computation of the gradient in \ref{eq:GD} is costly, due to the summation over all elements of $S$. The stochastic version (SGD) is cheaper. In each iteration a fixed sized mini batch $I_t\subseteq \{1,\dots,m\}$ is chosen randomly, and the update is given by

\begin{align*}
W^{(t+1)}=w^{(t)}-\eta_t\nabla_W\parenth{\sum_{i\in I_t}\ell(f(W^{(t)};x_i),y_i)}
\end{align*}

There are many useful variations for gradient descent. We use stochastic gradient descent (SGD) with Nesterov's momentum \cite{Nesterov}, which is very popular for CNN optimization. In this method there is an auxiliary vector $Z^{(t)}$ and an additional scalar learning parameter called the momentum coefficient, denoted by $\mu$. The update is given by

\begin{align*}
Z^{(t+1)}&=\mu Z^{(t)}-\eta_t\nabla_W\parenth{\sum_{i\in I_t}\ell(f(W^{(t)}+\mu Z^{(t)} ;x_i),y_i)}\\
W^{(t+1)}&=w^{(t)}+Z^{(t+1)}
\end{align*}

These methods are general, and can be applied to any function. However, theoretical guarantees exist only for convex functions. The loss functions of neural networks are non convex, but empirical studies have shown that such algorithms work pretty well in this framework. In non convex cases, the initial value of $W^{(0)}$ affects the performance. Common initialization schemes are randomized, for an example consult \cite{Glorot2010}. The initial value $Z^{(0)}$ is set to an all zeros vector.

A popular way for computing the needed gradients in CNNs, is the back propagation algorithm. A detailed explanation of the algorithm is given, for example, in chapter 6 of \cite{BengioBook}. In chapter \ref{comp_backprop} we give a detailed derivation of the variation fit for our model.

\newpage
\section{Motivations - Complex Numbers and Natural Images} \label{motivation}

Our main goal in this work is to construct a complex valued CNN. This idea stems from the fact that CNNs have proven to be very powerful in handling images, and that complex numbers can produce meaningful representations in this domain. In this section we describe different works that discuss ideas in similar directions, and how they motivate us to pursue the complex CNN model.

\subsection{Complex Valued ANNs} \label{comp_anns}

As early as the 1990's there have been attempts to construct complex valued neural networks, for example in \cite{Leung1991,Georgiou1992,Kim2002,Kim2003}. The main motivation behind these attempts was the observation that real valued data is often best understood when embedded in the complex domain. For example, waves are meaningfully represented by their Fourier coefficients.

In these works, the authors use artificial neural networks. They point at the problematic issues with introducing complex values into ANNs, and suggest solutions. These difficulties mainly focus on activation functions and the optimization problem. We will use some of these results in chapter \ref{construction}. The overall conclusion is that complex networks are comparable to the real valued networks in their performance, but there are numerical difficulties in training them.

None of these works discuss CNNs. Many practical methods were developed to allow better training for CNN models, raising hope that the numerical difficulties could be overcome in a CNN framework. Moreover, none of these works handle images, which could greatly benefit from complex representations, as shown in the following sections.

\subsection{Scattering Networks}\label{wavelets}

A Scattering network, first presented in \cite{Bruna2013}, is a restricted type of network  that provides a very good image representation. Using this representation the authors have achieved state of the art results for handwritten digits and texture classification. These networks are based on cascading the wavelet transform in different scales.

A wavelets family\footnote{For a more detailed explanation consult \cite{mallat2008wavelet}.} is composed of a concentrated waveform, and the translations and dilations of it. Waveforms are compactly represented in the complex domain, and so many wavelet families are composed of complex valued functions.
For every wavelet family, an image can be represented by its convolution with every function from the family. These wavelet features serve as a building block for many algorithms in computer vision.

Mallat and Bruna  have extended this idea by constructing scattering networks. In these networks there are alternating layers of convolution with wavelet functions, and the absolute value operator. Each layer outputs a local averaging of its values, and the aggregated outputs serve as a representation. These networks have gained considerable popularity due to their success. 

These networks are based on convolutions, but differ from CNNs in several aspects. First, there are no learned parameters, as the convolution kernels are predetermined wavelet functions. A recent work \cite{bruna2015theoretical}, suggests a similar, data-driven network with the same architecture but learned kernels.  This work is only theoretical, and there haven't been any empirical results yet. 
Second, the kernels are complex valued. However, since every convolution layer is followed by an absolute value operation, the propagating signal never remains complex. 

Given the interest in scattering networks it is only natural to examine what happens if these two constraints are loosened, i.e. if we allow the network to be fully complex, and learn the kernels in a data driven fashion. Our complex CNN model presents these two properties.

\subsection{Synchronization}\label{synchronization}

Despite the recent successes of neural networks, they are still outmatched by the human brain. Many of the processes taking place in the brain are not manifested in the simplified model of neural networks. Thus, a key question is weather any of these processes might allow neural networks to better handle complicated tasks. One candidate mechanism is the synchronization of neuronal signals. 

Neuronal rhythms are prevalent throughout the brain, and suspected to be important for neuronal communication. These are rhythmic patterns of neuronal spikes, i.e. peaks in the neuron's action potential.  Such rhythms are characterized by their average firing rate, and their phase. In conventional neural networks each neuron's output is represented by a single real valued scalar. This suggests an interpretation where each signal is represented only by its average firing rate.
However, relative phase between rhythmic signals might influence the resulting communication. Consider figure \ref{fig:sync} for an example of this effect. These are simulated neuronal rhythms, that are hypothesized to have a key role in neuronal communication. This figure demonstrates how the output rhythm depends not only on the input rates, but also their respective phase.

\begin{figure}[h!]
		\centering
		\includegraphics[width=1\textwidth]{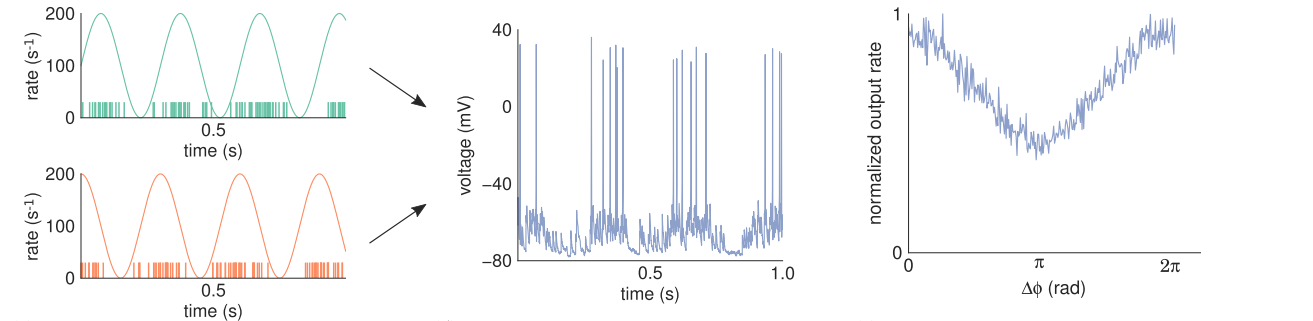}
		\caption{ The two rhythmic signals in the left pane are the input to the neuron whose output is presented in the middle pane. The two inputs have identical average firing rate, and differ in their phase. In the right panel, the graph shows how the resulting output rate depends on the phase difference between the two input signals. Taken from \cite{Reichert2013}. \label{fig:sync}}
\end{figure}

There have been numerous attempts to introduce synchronization into a neural network framework. In  \cite{Reichert2013} and \cite{RaoRavishankar2008}, the authors use complex numbers for this purpose. Based on Boltzmann Machines, they multiply the neuron values by a phase factor $e^{\i\theta}$. The activation function is modified accordingly, and composed of two terms. The classical term which is not affected by the phase, and a new term that is. With this activation function the output rate depends both on the input rates and their relative phases. 

The authors analyze this model and show that it manifests some known effects in neuroscience. One example is grouping, where neurons which respond to the same object share a common phase. They also show this empirically in several experimental setups. This phenomena demonstrates the potential importance of synchronization to computer vision tasks, such as semantic segmentation.

Another example of synchronization and its potential contribution for computer vision is seen in the recent work \cite{WermanComplexHough}. In this work the authors improve the well known Hough transform for finding circles, by introducing complex numbers. The traditional Hough transform is based on the fact that pixels in a circle all have gradients that point towards the circle's center. Each pixel in the image votes in the direction of its gradient, and the votes are accumulated. Pixels with a high score are potential circle centers.

In the variation presented in \cite{WermanComplexHough}, the votes are multiplied by a phase factor that depends on the distance. If $x$ is the voting pixel, then the score for every pixel $x'$ would be multiplied by $e^{\i C\abs{x-x'}}$ for some constant $C$. Votes coming from pixels on the same circle have the same distance to the center, so they have the same phase and their accumulated score is high. Votes originated in noise will typically have non synchronized phases, and will cancel each other out. The authors have demonstrated that this modification  yielded much cleaner results.

Both these works suggest that complex numbers can induce synchronization effects, which can benefit image related tasks. In our work we wish to create similar effects in CNNs.

\subsection{Restricting the Hypothesis Class as a Regularization Method}
\label{regularization_motivation}

It has been proved that CNNs are universal learners (cf. \cite{Maxout,Cohen2015}) i.e. they can implement practically any possible function.  A hypothesis class that is so expressive has a low approximation error, but a high risk of overfitting, which is a major difficulty in training CNNs. Many regularization methods have been designed to overcome overfitting.

Regularization methods can be roughly split into two categories. Methods that are aware of the data and problem at hand, and methods that aren't. In the first group, the methods are general, and can be applied to any CNN. One example, is the weight decay method (cf. chapter 7 in \cite{BengioBook}), which is a general technique in machine learning. The intuition behind this method is the Occam's razor principle, simple models are preferable to complex ones. The learning algorithm minimizes a term that measures the "complexity" of $f$ in addition to the empirical loss (more details can be found in \cite{ShalevBook}). One common measure is the squared $\ell_2$ norm of the weights' vector. 
A specific method for neural networks inspired by the same idea demands that the matrices in affine layers be of low rank \cite{Sainath2013}. 

Two popular methods for regularizing CNNs are dropout \cite{dropout} and dropconnect \cite{Wan2013}. These methods take advantage of the fact that the typical learning algorithm for CNNs is iterative. When applying dropout or dropconnect, a group of neurons or connections in a specific layer is zeroed out during training. The zeroed out group is randomly chosen in each iteration of the learning algorithm. It has been shown, e.g. in \cite{Wager2013}, that this mechanism introduces noise to the training set, reducing the risk of overfitting. 

Reducing the number of parameters decreases the risk of overfitting, but might increase the approximation error. Regularization methods from the second group exploit some prior knowledge about the data to construct a more compact model, without harming the approximation ability of the hypothesis class. An obvious case is the CNN class itself, which is a special case of ANNs suited for images. Subclasses of CNN have been developed for more specific tasks. One example is locally connected layers created to improve face recognition. In the architecture presented in \cite{Taigman2014}, locally connected layers replace some of the fully connected layers. These are restricted fully connected layers, where each neuron is affected only by its neighbors. Another example is the adjustment of CNNs for handling video streams, where the temporal dependencies between frames is exploited \cite{Ji2013}.

Generally speaking, the methods using prior knowledge are superior. First, because they result in more compact models preferable not only for regularization, but also for real-world applications' requirements. Second, as they make assumptions about the data, the resulting models are often more  interpretable.

We claim that a complex valued CNN can be seen as a regularization method of the second group. Any complex computation can be implemented as a real computation with more variables. We suggest that the restriction to complex calculations of a smaller model, fits the properties of images and certain problems, see \ref{regularization}. Thus, it might serve as a regularization method in these scenarios.

\clearpage
\section{Building a Complex Neural Network}\label{sec:Building_complex}
Convolutional Neural Networks produce the state of the art results for many computer vision tasks. A lot of  work and thought has been put into the CNN model and its specific details to make it work so well (e.g. \cite{LecunBengioHandbook,lecun1998gradient}). This success has prompted many attempts at expanding and improving this model. Inspired by the motivations presented in chapter \ref{motivation}, we consider complex valued CNNs where both inputs and weights are complex. We build our model as a generalization to the real model, with the hope of applying the known practices and shared beliefs about CNNs to its complex valued variation. 

We start by laying down some needed background from complex functions theory in section \ref{comp_bg}. Some of these functions' properties impose difficulties both in the construction of the network and in its optimization. These problems are presented in section \ref{construction}, along with possible solutions. Finally, the full derivation of complex valued gradient descent, and specifically back propagation, is presented in \ref{comp_backprop}.

\subsection{Complex Calculus - Preliminaries }\label{comp_bg}
We start by stating some known results from complex functions theory. Throughout this section we use the following notations for complex numbers:
\begin{align*}
z &=x+\i y\in\bC & x,y\in\reals\\
\end{align*}

And for complex functions:
\begin{align*}
f&:\mathbb{C} \to\mathbb{C}&\\
f(z)&=u(z)+\i v(z)& u,v:\reals\to\reals\\
\end{align*}

First, we point out that the complex field $\bC$ cannot be ordered in a meaningful way, i.e. there is no total ordering of $\bC$ under which the axioms of an ordered field are met. One implication is that the loss function we wish to minimize has to be real valued. To that end, we follow with some needed background about real valued complex functions. We focus on differentiability, as it plays a key role in the optimization process.

\begin{definition}
A complex function $f$ is complex differentiable at $z$, with the derivative $f'(z)$, if the following limit exists\[
f'(z)=\lim_{h\to 0}\frac{f(z+h)-f(z)}{h}
\]
\end{definition}
A function that is complex differentiable everywhere is called \textit{entire}. A very useful equivalent definition is given by the Cauchy-Riemann equations.

\begin{definition}
A complex function $f$ is \textit{complex differentiable} at point $z$ if and only if $u,v$ are differentiable (as real functions) there, and the Cauchy-Riemann equations hold at $z$:
$$
\frac{\partial u}{\partial x} = \frac{\partial v}{\partial y},
\frac{\partial u}{\partial y} = -\frac{\partial v}{\partial x}
$$
\end{definition}

Complex differentiability is a very strong property, much stronger then its real equivalent. For example, if $f$ is real valued, namely $f(z)=u(z)$, then the CR equations reduce to\[
\pderiv{u}{x}=\pderiv{u}{y}=0
\]
If such an $f$ is entire, it is constant.

Another result implied by the above is the Liouville Theorem which states that an entire function that is bounded everywhere is constant.

In the following part of this section we present the Wirtinger derivatives, which will be used to adjust gradient based methods to the complex domain. First, we define the differentials with respect to the variables $z$ and its conjugate $z^*$ :
\begin{definition}\label{diffs}
	\begin{align*}
		\diff z=\diff x+\i \diff y \\
		\diff z^*=\diff x-\i \diff y 
	\end{align*}
\end{definition}
These differentials impose partial derivatives, which are called Wirtinger derivatives.
\begin{definition}\label{Wirtinger}
	The Wirtinger derivatives operators are
	\begin{align*}
		\frac{\partial}{\partial z}:=\frac{1}{2}\left[\frac{\partial}{\partial x} - \i \frac{\partial}{\partial y}\right]\\
		 \frac{\partial}{\partial z^*}:=\frac{1}{2}\left[\frac{\partial}{\partial x} + \i \frac{\partial}{\partial y}\right]
	\end{align*}
\end{definition}

The Wirtinger derivatives have some desirable properties. For one, $z,z^*$ are independent variables as\[
\frac{\partial z}{\partial z^*} =\frac{\partial z^*}{\partial z}=0\]
Also, some dual connections with the conjugate hold for the derivatives as well,
\begin{align}\label{conj}
\frac{\partial f^*(z)}{\partial z} =\left(\frac{\partial f(z)}{\partial z^*}\right)^* ,\quad \left(\frac{\partial f(z)}{\partial z}\right)^* =\frac{\partial f^*(z)}{\partial z^*}
\end{align}
Using the Wirtinger derivatives, we can express the total differential of any complex valued function $f$.
\begin{theorem}\label{tot_diff}
	The differential $\diff f$	of a complex-valued	function $f(z):\bA\to\bC$ with $\bA\subseteq\bC$ can be expressed as 
	$$
	\diff f=\frac{\partial f(z)}{\partial z}dz+\frac{\partial f(z)}{\partial z^*}dz^*
	$$
\end{theorem}
\begin{proof}
Consider the bivariate functions $F:\reals^2\to \bC$ and $U,V:\reals^2\to\reals$ associated to $f(z)$ by\[
\forall z=x+\i y,\quad F(x,y)=U(x,y)+\i V(x,y)=f(z)
\]
The total differential of $F$ is given by
\begin{align}\label{F_tot_diff}
\diff F = \pderiv{F}{x}\diff x+\pderiv{F}{y}\diff y =  \pderiv{U}{x}\diff x+\i  \pderiv{V}{x}\diff x+\pderiv{U}{y}\diff y+\i\pderiv{V}{y}\diff y
\end{align}
By using the differentials defined above, we can write\[
\diff x = \frac{1}{2}\parenth{\diff z +\i \diff z^*},\quad \diff y = \frac{1}{2\i}\parenth{\diff z -\i \diff z^*}
\]
Obtaining
\begin{align*}
\diff F &= \frac{1}{2}\brac{\pderiv{}{x}\parenth{U+\i V}-\i\pderiv{}{y}\parenth{U+\i V}}\diff z+ \frac{1}{2}\brac{\pderiv{}{x}\parenth{U+\i V}+\i\pderiv{}{y}\parenth{U+\i V}}\diff z^*=\\
&=\frac{1}{2}\brac{\pderiv{F}{x}-\i\pderiv{F}{y}}\diff z+ \frac{1}{2}\brac{\pderiv{F}{x}+\i\pderiv{F}{y}}\diff z^*= \frac{\partial f}{\partial z}dz+\frac{\partial f}{\partial z^*}dz^*
\end{align*}
\end{proof}

Considering a a real valued function $f:\bA\to\reals$ for some $\bA\subseteq\bC$, its total differential can be expressed using the Wirtinger derivatives, as seen in the following theorem.

\begin{theorem}\label{real_diff}
Let  $\bA\subseteq\bC$, and $f:\bA\to\reals$ be a real valued function. The total differential of $f$ is given by \[
\diff f=2\Re\parenth{\pderiv{f}{z}\diff z}=2\Re\parenth{\pderiv{f}{z^*}\diff z^*}
\]
\end{theorem}
\begin{proof}
	From definitions \ref{diffs} and \ref{Wirtinger} of the Wirtinger differentials and partial derivatives we obtain\[
	\pderiv{f}{z}\diff z = \frac{1}{2}\parenth{\pderiv{f}{x}-\i\pderiv{f}{y}}\parenth{\diff x+\i \diff y}
	\]
	Hence\[
	2\Re\parenth{\pderiv{f}{z}\diff z} = \pderiv{f}{x}\diff x+\pderiv{f}{y}\diff y
	\]
	The analogue statement holds for the conjugates \[
	2\Re\parenth{\pderiv{f}{z^*}\diff z^*} = \pderiv{f}{x}\diff x+\pderiv{f}{y}\diff y
	\]
	
	From the definition of the total differential in equation \ref{F_tot_diff}, for the associated $U$,
	\begin{align*}
	\diff f=\pderiv{U}{x}\diff x +\pderiv{U}{y}\diff x =  \pderiv{f}{x}\diff x+\pderiv{f}{y}\diff y 
	\end{align*}
	Which concludes the proof.
\end{proof}

From theorem \ref{real_diff} we can deduce the following,
\begin{corollary}\label{steepest}
	For the aforementioned $f$, the steepest ascent at point $z$ is obtained by \[
	\diff z =\pderiv{f}{z^*}\diff s
	\]
	Where $\diff s$ is a real-valued differential. Therefore, the steepest ascent's direction is \[
	\pderiv{f}{z^*}
	\]
\end{corollary}
\begin{proof}
	According to theorem \ref{real_diff} \[
	\diff f=2\Re\parenth{\pderiv{f}{z^*}\diff z^*}
	\]
	Thus, for a fixed norm $\diff z$, $\diff f$ is maximized when $\pderiv{f}{z}\diff z$ is real, i.e. $\diff z$ is a scaled version of $\parenth{\pderiv{f}{z}}^*=\pderiv{f}{z^*}$, where the equality is obtained by applying equation \ref{conj} for a real valued $f$. Equivalently, $dz^* $ is a scaled version of $\pderiv{f}{z^*}$ which concludes the proof.
\end{proof}

We use this corollary in section \ref{comp_backprop}, where we tackle the challenge of optimizing complex valued CNNs.

\subsection{Network Structure}\label{construction}

We build our complex model as a generalization of real valued CNNs, which handles complex valued input and weights. Many of CNN's building blocks generalize trivially, but for some, the lack of ordering of the complex field makes the generalization tricky. Without total ordering, two general complex numbers are not comparable, and specifically the $\max$ and $\min$ operators are not defined. $\relu$, $\max$ pooling and the optimization problem itself all rely on these operators. In this section we suggest possible generalizations for these building blocks and discuss their pros and cons.

\subsubsection{ReLU}
The most common activation function used by the CNN community is the rectified linear unit, or $\relu$. To avoid confusion, we will refer to this function by $\relu_\Re$ in this section. 
\begin{definition*} 
	$\forall x\in \reals \quad \relu_\Re(x)=\begin{cases} x & x\geq 0\\0 & o.w. \end{cases}$
\end{definition*}

To stay as close as possible to the real model, we construct the complex $\relu$ in the same manner as its real value counterpart. For some connected $A\subseteq\bC$
\begin{align*}
\relu(z)=\begin{cases} z & z\in A\\0 & o.w. \end{cases}\\
\end{align*}
For a complex function $\relu(z)$ to generalize $\relu_\Re$ it should obey 
\[
\forall x\in\reals \; \relu(x)= \relu_\Re(x)\]
Which reduces to 
\begin{align*}
\{z\mid \Im(z)=0,\Re(z)\geq 0\}\subseteq A\\
 \{z\mid\Im(z)=0,\Re(z)< 0\}\not\subseteq A
\end{align*}

Following Occam's razor, the simplest choice is a sector containing the positive real ray. In such a case $A$ can be written as $\{z\mid \arg(z)\in \left[\theta_1,\theta_2\right]\}$ for some $-\pi<\theta_1\leq0\leq\theta_2<\pi$. 
The value of $\theta_2-\theta_1$ controls what portion of the plain is zeroed out. $\theta_1,\theta_2$ can be set in advance, or learned via cross validation. Unfortunately, $\relu$ is not derivable w.r.t  $\theta_1,\theta_2$ so it cannot be learned during the training process like other parameters.
 \remove{In the real $\relu_\Re$ one half of the real line is passed.}  

As the $\relu_\reals$ passes only positive values, an intuitive generalization is to pass values with positive real and imaginary parts. In the above notations this translates to setting $\theta_1=0,\theta_2=\frac{\pi}{2}$, resulting in

	\begin{align*}
		\relu(z)=
		\begin{cases}
		z & \Re\left(z\right),\Im\left(z\right)\geq0\\
		0 & o.w.
		\end{cases}
		=
		\begin{cases}
		z & \arg (z)\in\left[0,\frac{\pi}{2}\right]\\
		0 & o.w.
		\end{cases}
	\end{align*}

\subsubsection{Pooling}\label{pooling}

In a pooling layer, the input is split into patches, and each patch is replaced by one value. In the popular max pooling, this value is the maximal value of the original patch - $\max_{z\in\text{patch}}{z}$. Since the $\max$ operator is not defined for complex numbers, it does not generalize trivially. We suggest two possible generalizations, max-by-magnitude which is based on projection, and max-by-softmax.

A simple way to compare values in $\bC$ is by comparing their projection to $\reals$. Natural projections include $\phi(z)=\abs{z},\angle{z},\Re(z),\Im(z)$. Given a projection $\phi:\bC\to\reals$ the complex valued max pooling is given by  $\argmax_{z\in\text{patch}}{\phi(z)}$. Using $\argmax$ instead of $\max$ is desirable for two reasons. First, it sets the output value to be one of the input values, similarly to the real valued case. Second, it enables the values of the network to stay complex throughout the computation, as we wish to allow in this model.

Considering the suggested projections, $\phi(z)=\abs{z}$ is the only reasonable choice. The $\Re(z),\Im(z)$ projections are not suitable for this purpose, as they favor one of the real and imaginary parts over the other, while the other operations in the network do not differentiate between them. The argument, $\arg(z)$, is periodic by nature, and so senseless for comparison purposes. The magnitude is a reasonable measure, and we suggest the max-by-magnitude pooling, given by
\[
\argmax_{z\in\text{patch}}{\abs{z}}\]

Max-by-magnitude pooling generalizes the real valued max pooling only for non negative inputs. If the patch contains positive and negative real values the results of the two might differ. For example, if the values in the patch are $\{-5,2\}$ then $\argmax_{z\in\text{patch}}{z}=2$ while $\argmax_{z\in\text{patch}}\abs{z}=-5$. However, in typical CNNs, a pooling layer follows a $\relu$ layer, which prevents this scenario.

Another possible approach of generalizing $\max$ pooling is based on presenting the $\max$ operator as a limit of parametrized functionals. If these functionals are well defined in the complex domain, they can be used for pooling in the complex case.

One possible family is the softmax \footnote{There are many definitions to the softmax operator. We use this definition because it is well defined for complex input.} functionals, defined by:
\begin{definition} [Softmax]
For every $\alpha\in \reals$, and $\{x_i\}_{i=1}^n\in\reals^n$
	$$
	\softmax_\alpha(\{x_i\}_{i=1}^n) = \frac{\sum_{i}x_{i}e^{\alpha x_{i}}}{\sum_{j}e^{\alpha x_{j}}}
	$$	
\end{definition}

By taking $\alpha$ to the limits of $\pm \infty$ and $0$ we obtain that for every $\{x_i\}_{i=1}^n\in\reals^n$
\begin{align*}
	\frac{\sum_{i}x_{i}e^{\alpha x_{i}}}{\sum_{j}e^{\alpha x_{j}}}&\to
	\begin{cases}
	\max_i{x_{i}} & \alpha\to\infty\\
	\frac{1}{n}\sum_{i}x_{i} & \alpha\to0\\
	\min_i{x_{i}} & \alpha\to-\infty
	\end{cases}		
\end{align*}

The different limits of the softmax can prove beneficial also to real valued networks, as they offer a smooth transition between max, average and min. Max and average pooling are both used in CNNs, and it is not always easy to predict which one will perform better. The ability to transfer smoothly between them might create some intermediate operator that would increase performance. Moreover, the parameter $\alpha$ can be learned in the training process, reducing some of the necessary cross validation between architectures.

This family generalizes naturally for complex inputs $\{z_i\}_{i=1}^n\in\bC^n$. Which induces max-by-softmax, for which the output for every patch is given by $\softmax_{z\in \text{patch}}z$.

To examine the limits for the complex case, let $z_i=x_i+\i y_i$ for every $i$, and denote $x_{i_0}=\max_ix_i$. By taking the limit of $\alpha\to\infty$ we obtain
\begin{align*}
	\softmax_\alpha(\{z_i\}_{i=1}^n) &= \frac{\sum_{i}z_{i}e^{\alpha z_{i}}}{\sum_{j}e^{\alpha z_{j}}}\\
	&=\frac{\sum_{i}r_{i}e^{\alpha x_{i}}e^{\i\left(\theta_{i}+\alpha y_{I}\right)}}{\sum_{j}e^{\alpha x_{j}}e^{\i\alpha y_{j}}}\\
	&=\frac{e^{\alpha x_{i_{0}}}\sum_{i}r_{i}e^{\alpha\left(x_{i}-x_{i_{0}}\right)}e^{\i\left(\theta_{i}+\alpha y_{I}\right)}}{e^{\alpha x_{i_{0}}}\sum_{j}e^{\alpha\left(x_{j}-x_{i_{0}}\right)}e^{\i\alpha y_{j}}}=\\&=\frac{r_{i_{0}}e^{\i\left(\theta_{i_{0}}+\alpha y_{I_{0}}\right)}+\sum_{i\neq i_{0}}\underbrace{e^{\alpha\left(x_{i}-x_{i_{0}}\right)}}_{\to0}r_{i}e^{\i\left(\theta_{i}+\alpha y_{i}\right)}}{e^{\i\alpha y_{i_{0}}}+\sum_{j\neq j_{0}}\underbrace{e^{\alpha\left(x_{j}-x_{i_{0}}\right)}}_{\to0}e^{\i\alpha y_{j}}}\\&\to r_{i_{0}}e^{i\theta_{i_{0}}}=x_{i_{0}}+\i y_{0}
\end{align*}
In a similar fashion, we obtain three limits, analogous to the real case:
\begin{align*}
	\softmax_\alpha(\{z_i\}_{i=1}^n) =&\to\begin{cases}
	\arg\max_{z_{i}}\Re\left(z_{i}\right) & \alpha\to\infty\\
	\frac{1}{n}\sum_{i}z_{i} & \alpha\to0\\
	\arg\min_{z_{i}}\Re\left(z_{i}\right) & \alpha\to-\infty
	\end{cases}		
\end{align*}

These limits share the flexibility proposed by the real softmax. However, they contain an inherent symmetry breaking between the real and imaginary parts. This is unwanted in the context of pooling, as discussed earlier in the context of the projections $\Re$ and $\Im$. We suggest a possible way to overcome this is by defining a "dual operator" defined by
\begin{definition}[Dual Softmax]
	$$
	\softmax^*_\alpha(\{z_i\}_{i=1}^n) = \softmax_\alpha(\{\i z^*_i\}_{i=1}^n)
	$$	
\end{definition}

The limits of the dual operator are similar to the $\softmax$'s limits with the imaginary part instead of the real part:
\begin{align*}
	\softmax^*_\alpha(\{z_i\}_{i=1}^n) =&\to\begin{cases}
	\arg\max_{z_{i}}\Im\left(z_{i}\right) & \alpha\to\infty\\
	\frac{1}{n}\sum_{i}z_{i} & \alpha\to0\\
	\arg\min_{z_{i}}\Re\left(z_{i}\right) & \alpha\to-\infty
	\end{cases}		
\end{align*}

A pooling layer can be constructed by a combination of the two, either by applying both in different channels or by using some linear combination of the two. 

\subsubsection{Projection Layer}\label{loss_real}

In many applications the labels are real valued, and so is the network's output. For example, in  a classification task with $k$ classes, the last layer of the network is typically a vector with $k$ entries. This vector is normalized to have positive values that sum up to one, and interpreted as a probability vector, where the $i^{th}$ coordinate's value is the probability that the input belongs to the $i^{th}$ class. Consequently, the output vector has to be real valued.

To that end we add a projection layer, which is a special case of an activation function layer. An obvious choice in many cases is projection by magnitude, for the same reasons discussed earlier. Numerical considerations which will be elaborated in the following sections suggest using the squared magnitude.

\subsection{Network Optimization - Complex Backpropagation} \label{comp_backprop}

The common way to train a neural network, i.e. to minimize its loss function $\ell(W)$, is by using the iterative gradient descent algorithm presented in \ref{intro:opt}. Starting with an initial value for $W$, at each iteration the weights are updated by adding a step in the direction of $\ell$'s steepest descent,  given by the opposite to the gradient. In the complex case, the loss function is real valued with complex weights. Such a function is not differentiable everywhere, and it's steepest descent direction cannot be calculated using the gradient. To that end we use the Wirtinger derivatives presented in section \ref{comp_bg}, and specifically the multivariate generalization of corollary \ref{steepest}.

Throughout this section we use the following notations regarding real valued multivariate functions. Given a scalar function $\ell$ , we denote its gradient with respect to its variables matrix $A$ by $\pderiv{\ell}{A}$. I.e. $\pderiv{\ell}{A}$ is a matrix, where for every index $[i,j]$\[
\pderiv{\ell}{A}[i,j]=\pderiv{\ell}{A[i,j]}
\]
Similarly, given a non scalar function $X_{t+1}$ with variables $X_{n}$ we denote the Jacobian of $X_{n+1}$ with respect to $X_n$ by $\pderiv{X_{n+1}}{X_n}$. I.e. for all possible indices $[p,q,i,j] $:\[
\pderiv{X_{n+1}}{X_n}[p,q,i,j]=\pderiv{X_{n+1}[p,q]}{X_n[i,j]}
\]

Denoting the complex valued weights by $W=A+\i B$, the multivariate generalization of corollary \ref{steepest} suggests that the gradient descent step should by taken in the direction
\begin{align}\label{comp_gradient}
-\parenth{\pderiv{\ell}{A}+\i \pderiv{\ell}{B}}
\end{align}
	
In neural networks, the gradient is typically computed using the backpropagation algorithm. In this section we describe the adapted backpropagation algorithm for calculating the derivatives of equation \ref{comp_gradient}.

Consider the $n^{th}$ layer of a complex valued network, with input, weights and output denoted by $Z_n,W_n$, and $Z_{n+1}$ respectively. Denote the real and imaginary parts by\[
Z_n=X_n+\i Y_n,\quad W_n=A_n+\i B_n
\]
Denote the derivatives of the loss with respect to the input by 
\begin{align*}
\delta_n=\delta_n^\Re+\i \delta_n^\Im=\pderiv{\ell}{X_n}+\i \pderiv{\ell}{Y_n}\\
\end{align*}
The backpropagation's output is the derivatives with respect to the weights, \[
\pderiv{\ell}{A_n}+\i \pderiv{\ell}{B_n}
\]

 The backpropagation algorithm is composed of two passes, forward and backward. In the forward pass, the values of each layer, $Z_n$, are computed according to the network's architecture, form the first layer to the final $N^{th}$ layer. In the backward pass, the final layer's gradient $\delta_N$ is computed, and then $\delta_n$ and $\pderiv{\ell}{A_n}+\i \pderiv{\ell}{B_n}$ are computed for every $n$ is reverse order, according to the chain rule. Finally, the algorithms output is the concatenation of $\pderiv{\ell}{A_n}+\i \pderiv{\ell}{B_n}$ for all layers.

In the following sections we present how to compute $\delta_n$, and $\pderiv{\ell}{A_n}+\i \pderiv{\ell}{B_n}$, for each type of layer, given $\delta_{n+1}$. Most of the following computations have been done in the past, for example in \cite{Kim2002}.

\subsubsection{Affine layer}
In an affine layer, the output is given by
\[
Z_{n+1}=W_nZ_n+\hat{w}_n\cdot \mathbf{1}^\intercal\\
\]
Where the weights are the matrix $W_n=A_n+\i B_n$ and the vector $\hat{w}_n=\hat{a}_n+\i\hat{b}_n$. When splitting to the real and imaginary parts, we obtain
\begin{align*}
X_{n+1}&=A_nX_n-B_nY_v+\hat{a}\cdot \mathbf{1}^T\\
Y_{n+1}&=A_nY_n+B_nX_v+\hat{b}\cdot \mathbf{1}^T
\end{align*}
Which yields the Jacobians
\begin{align*}
\pderiv{X_{n+1}}{X_n}[p,q,i,j]&=A_n[p,i]1_{\brac{q=j}}, \quad \pderiv{X_{n+1}}{Y_n}[p,q,i,j]=-B_n[p,i]1_{\brac{q=j}}\\
\pderiv{Y_{n+1}}{X_n}[p,q,i,j]&=B_n[p,i]1_{\brac{q=j}}, \quad \pderiv{Y_{n+1}}{Y_n}[p,q,i,j]=A_n[p,i]1_{\brac{q=j}}
\end{align*}
Where $1_{\brac{q=j}}=\begin{cases}1 & q=j\\0 & o.w.\end{cases}$.

Applying the chain rule for every index $[i,j]$ yields 
\begin{align*}
\delta_n^\Re [i,j]&= \sum_{pq}\delta^\Re_{n+1}[p,q]A_n[p,i]1_{\brac{q=j}}+\delta^\Im_{n+1}[p,q]B_n[p,i]1_{\brac{q=j}} =\\ &=\parenth{(W_n^\Re)^\intercal\delta_{n+1}^\Re+(W_n^\Im)^\intercal\delta_{n+1}^\Im}[i,j]
\\
\delta_n^\Im [i,j]&= \pderiv{\ell}{Y_n}= \sum_{pq}\delta^\Re_{n+1}[p,q](-B_n)[p,i]1_{\brac{q=j}}+\delta^\Im_{n+1}[p,q]A_n[p,i]1_{\brac{q=j}} = \\
&=\parenth{-(W_n^\Im)^\intercal\delta_{n+1}^\Re+(W_n^\Re)^\intercal\delta_{n+1}^\Im}[i,j]
\\
\end{align*}
Which reduces to the compact form 
\begin{align}
\delta_n = \delta_n^\Re+\i \delta_n^\Im=W_n^H\delta_{n+1}
\end{align}
Where $W_n^H$ is the hermitian conjugate of $W_n$, i.e. for every $i,j$, $W_n^H[i,j] =\overline{W_n[j,i]}$.

Applying the same technique over the weights yields
\begin{align}
\pderiv{\ell}{A_n}+\i\pderiv{\ell}{B_n} = \delta_{n+1}Z_n^H\\
\pderiv{\ell}{\hat{a}_n}+\i\pderiv{\ell}{\hat{b}_n} = \delta_{n+1}\cdot \mathbf{1}
\end{align}

\subsubsection{Activation Function Layer}

In an activation function layer, with the function $f=u+\i v$, the output in index $[i,j]$ is given by \[
Z_{n+1}[i,j]=f(Z_n[i,j])=u(Z_n[i,j])+\i v(Z_n[i,j])
\]
Which translates to \[
X_{n+1}[i,j]=u(Z_n[i,j]),\quad Y_{n+1}[i,j]=v(Z_n[i,j])
\]
Hence the Jacobians are
\begin{align*}
\pderiv{X_{n+1}}{X_n}[p,q,i,j]&=\pderiv{u(Z_n[i,j])}{X_n[i,j]}\cdot1_{\brac{q=j,p=i}}\\
\pderiv{X_{n+1}}{Y_n}[p,q,i,j]&=\pderiv{u(Z_n[i,j])}{Y_n[i,j]}\cdot1_{\brac{q=j,p=i}}\\
\pderiv{Y_{n+1}}{X_n}[p,q,i,j]&=\pderiv{v(Z_n[i,j])}{X_n[i,j]}\cdot1_{\brac{q=j,p=i}}\\
\pderiv{Y_{n+1}}{Y_n}[p,q,i,j]&=\pderiv{v(Z_n[i,j])}{Y_n[i,j]}\cdot1_{\brac{q=j,p=i}}\\
\end{align*}
 And the derivatives reduce to 
 \begin{align*}
 \delta_n^\Re [i,j] &=  \pderiv{\ell}{X_n}=\delta^\Re_{n+1}[i,j]\pderiv{u(Z_n[i,j])}{X_n[i,j]}+\delta^\Im_{n+1}[i,j] \pderiv{v(Z_n[i,j])}{X_n[i,j]}\\
 \delta_n^\Im [i,j]&= \pderiv{\ell}{Y_n}= \delta^\Re_{n+1}[i,j]\pderiv{u(Z_n[i,j])}{Y_n[i,j]}+\delta^\Im_{n+1}[i,j] \pderiv{v(Z_n[i,j])}{Y_n[i,j]} \\
 \end{align*}
Combining the real and imaginary parts yields
 \begin{align}
 \delta_n[i,j]= \delta^\Re_{n+1}[i,j]\parenth{\pderiv{u(Z_n[i,j])}{X_n}+\i\pderiv{u(Z_n[i,j])}{Y_n}}+\\
 \i\delta^\Im_{n+1}[i,j]\parenth{\pderiv{v(Z_n[i,j])}{Y_n}-\i\pderiv{v(Z_n[i,j])}{X_n}}
 \label{eq:af_backprop_complex}
\end{align}
If $f$ is complex differentiable, this translates to a compact form
\begin{equation}
\delta_n[i,j]=\delta_{n+1}[i,j]f'(Z_n[i,j])^*
\label{eq:af_backprop}
\end{equation}

Naively, using the compact form requires that $f$ be entire. However, it is practically sufficient that the update will be correct for a very large portion of the iterations, and that when it doesn't, it will have a finite value. If these conditions are met, the convergence of the iterative algorithm should not suffer.
The $\relu$ activation function meets this condition. It is differentiable everywhere but at $\{z|\Re(z)=0 \text{ or } \Im(z)=0\}$, where the limits are finite.

In the case of a projection activation function layer, $f=u$ is real valued, and so non differentiable. In this case equation \ref{eq:af_backprop_complex} takes the form
\begin{align}
\delta_n[i,j]= \delta^\Re_{n+1}[i,j]\parenth{\pderiv{u(Z_n[i,j])}{X_n}+\i\pderiv{u(Z_n[i,j])}{Y_n}}
\end{align}

\subsubsection{Convolution Layer}

Each output value of a convolution layer is a dot product between a kernel, and an input patch. Thus, if the input is reorganized as a matrix, with each column being one patch, and the weights are organized as a matrix, with one kernel in each row, the convolution is the multiplication of the two matrices.

For the purpose of backpropagation, it is more convenient to use the above observation and express a convolution layer as a composition of three layers: A reorganization layer (to matrix form), an affine layer, and another reorganization layer. The reorganization layers do not change any values, but only their locations.

The backpropagation of reorganization layers is very simple. Let $[i,j]$ be an input index, which is moved by the reorganization layer to the new indices $[i_1,j_1],\dots,[i_d,j_d]$ then\[
\delta_n[i,j]=\sum_{k=1}^d\delta_{n+1}[i_k,j_k]
\]

\subsubsection{Pooling Layer}
A max-by-pooling layer can be represented similarly to the convolution layer by a composition of a reorganization layer, an operation over each column, and another reorganization layer. 

In the case of max-by-magnitude pooling, the value $[i,j]$  is transfered to the output if it has the maximal magnitude in its column. Denote its index in the output by $[p,q]$ then  \[
\delta_n[i,j]=\delta_{n+1}[p,q]
\]
If the value at index $[i,j]$ did not transfer to the output then\[
\delta_n[i,j]=0\]

Softmax pooling can be similarly constructed as a combination of reorganization, affine and activation function layers.

\subsection{Complex Convolution as a Restricted Real Convolution method}\label{regularization}

\begin{wrapfigure}{R}{0.4\textwidth}
	{\includegraphics[width=0.4\textwidth]{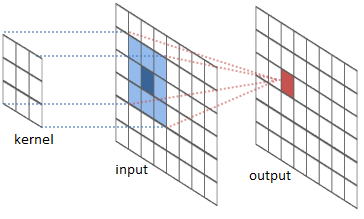}
		\caption{A schematic sketch of the convolution operation. An item in the output is the sum of point-wise multiplication of the kernel and input patch.\label{fig:conv_sketch}}}
	\end{wrapfigure}
A real valued convolution operation takes a matrix and a kernel (a smaller matrix), and outputs a matrix. The matrix elements are computed using a sliding window with the same dimensions as the kernel. Each element is the sum of the point-wise multiplication of the kernel and matrix patch at the corresponding window.  Figure \ref{fig:conv_sketch} shows a schematic representation of a convolution.

We will use here the dot product to represent the sum of a point-wise multiplication between two matrices:

\[
X\cdot A = \sum_{ij}X_{ij}A_{ij}
\]

In the complex generalization, both kernel and input patch are complex valued. The only difference stems from the nature of multiplication of complex numbers. When convolving a complex matrix with the kernel $W=A+\i B$, the output corresponding to the input patch $Z=X+\i Y$ is given by
\begin{align}
Z\cdot W = \parenth{X\cdot A-Y\cdot B}+\i\parenth{X\cdot B+ Y\cdot A}\label{eq:comp_patch}
\end{align}

To implement the same functionality with a real valued convolution, the input and output should be equivalent. Each complex matrix is represented by two real matrices, stacked together in a three dimensional array. Denoting this array $[X,Y]$, it's equivalent to $X+\i Y$. $X$ and $Y$ are the array's channels.

A two channeled input, convolved with a two channeled kernel, results in a one channeled matrix. The dot product between a kernel $[A,B]$ and an image patch $[X,Y]$ is given by: 
\begin{align*}
X\cdot A+Y\cdot B
\end{align*}
Convolution with multiple kernels produces multiple channels. Specifically, when convolving with two kernels $[A_1,B_1],[A_2,B_2]$, the output corresponding to the patch $[X,Y]$ is given by
\begin{align}
[X\cdot A_1+Y\cdot B_1,X\cdot A_2+Y\cdot B_2]\label{eq:real_patch}
\end{align}

Comparing equations \ref{eq:comp_patch} and \ref{eq:real_patch}, it is clear that given a complex convolution with kernel $A+\i B$, an equivalent real convolution has two kernels of the form $[A,-B]$ and $[B,A]$, as seen in figure  \ref{fig:sketch_conv_comp_and_real}.

\begin{figure}[h!]
		\centering
	\begin{subfigure}[Complex valued convolution]
		{\includegraphics[width=0.45\textwidth]{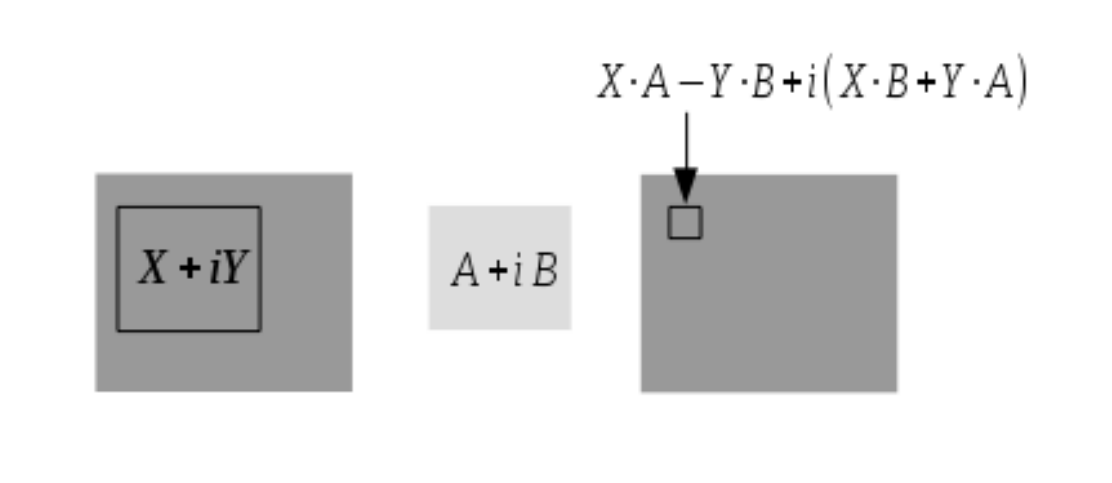}\label{fig:sketch_comp_conv}}
	\end{subfigure}
	\quad
	\begin{subfigure}[Real valued convolution]
		{\includegraphics[width=0.45\textwidth]{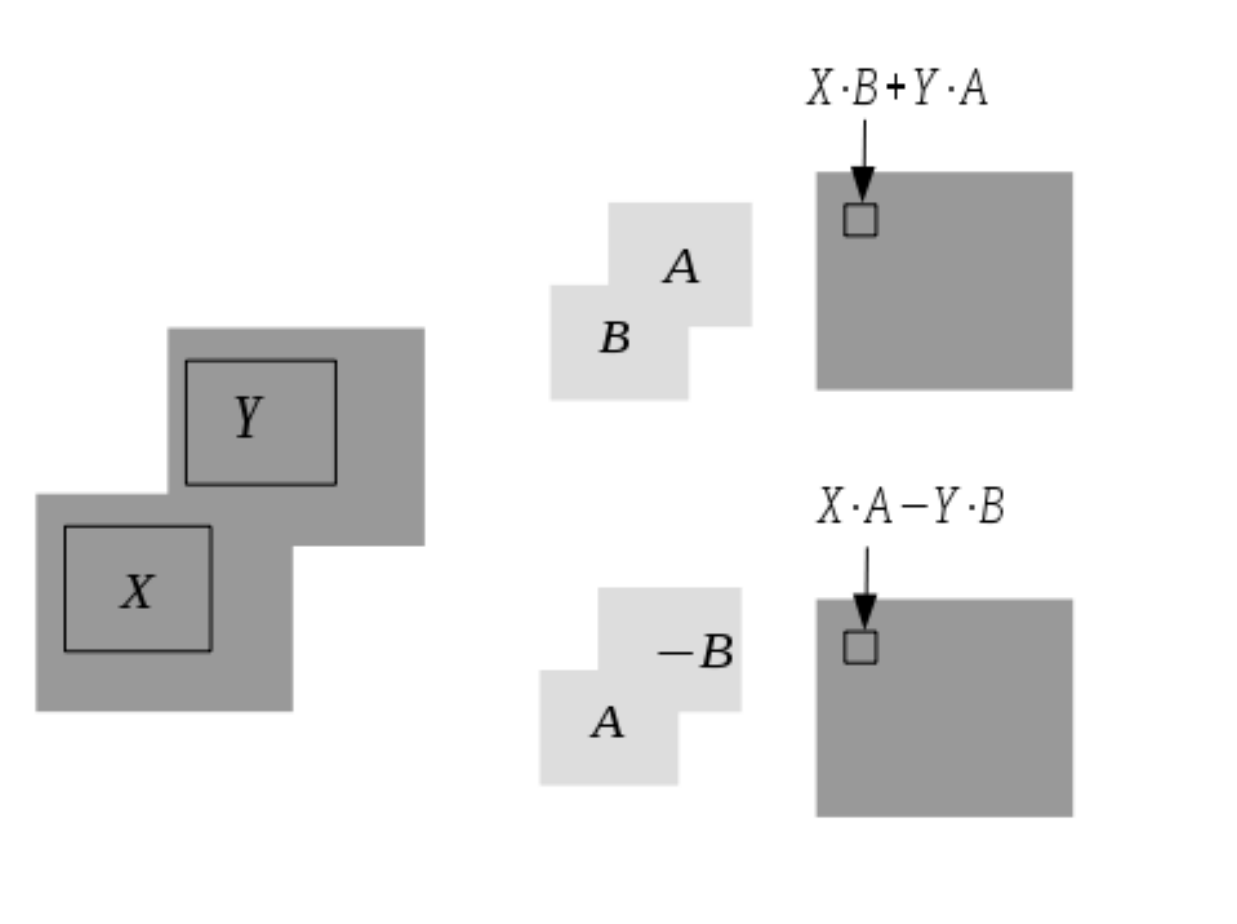}\label{fig:sketch_real_conv}}	
	\end{subfigure}					
	\caption{Equivalent complex and real convolution layers. \subref{fig:sketch_comp_conv} Complex valued convolution, where the output pixel is given by the dot product of the input patch and the kernel. \subref{fig:sketch_real_conv} Equivalent real valued convolution with two channeled input, output and kernels. Convolving with one kernel yields one channel.\label{fig:sketch_conv_comp_and_real}}
\end{figure}

In light of the above, a convolution layer in a complex valued network is a restricted form of a real valued convolution layer with twice as many kernels. 

We note that the equivalence between real and complex networks does not hold in non affine layers. Activation function and pooling layers operate on one channel, so the real valued equivalent layers should operate on two channels, which is not the case. In these layers the complex network can be seen as connecting the channels, rather than the weights.

\subsection{Complex Convolution} \label{comp_conv}
The previous result raises the question for which case is a complex CNN a good classifier. In order to answer this question we analyze the real and complex convolutions.

A real convolution output can be interpreted as a heat map of similarity to the convolved kernel. This view is based on the interpretation of the dot product between two matrices as a similarity measure. Indeed, a dot product between a real patch and a kernel with norm 1, is maximized when they are identical up to a scalar\footnote{The norm over matrices is defined by $\norm{A}=\sqrt{A\cdot A}$, the norm of the vectorized matrix. The dot product $X\cdot A$ scales together with the norm of $A$, therefore the maximization considers only norm 1 kernels.}, i.e.\[
	\argmax_{\norm{A}=1} X\cdot A  =\frac{X}{\left\Vert X\right\Vert }
	\]

To better understand the complex convolution we look at the equivalent complex valued optimization problem. As $Z\cdot W$ is a complex number, we maximize its magnitude,
\[
	\argmax _{\norm{W}=1} \abs{Z\cdot W}
	\]
Denoting \[
\forall i,j\quad Z_{ij}=r_{ij}e^{\i\theta_{ij}}, W_{ij}=t_{ij}e^{\i\nu_{ij}}\]
We obtain the maximization problem
\begin{align*}\label{eq:opt_comp_conv}
\argmax_W \abs{\sum_{ij} r_{ij} t_{ij}e^{\i(\theta_{ij}+\nu_{ij})}}
\end{align*}

A geometric interpretation is given by thinking of each complex number as a two dimensional vector. In this view, multiplying $Z_{ij}$ by $W_{ij}$ rotates it by an angle $\nu_{ij}$. The sum of the rotated vectors has maximal magnitude if they all have the same phase and their magnitudes accumulate, otherwise the summed terms cancel each other out. Therefore, the maximizing kernel obeys\[
\forall i,j\quad \nu_{ij}=-\theta_{ij}+C
\]
Or equivalently,\[
W=e^{\i C}\frac{Z^*}{\left\Vert Z\right\Vert }
\]
Where $Z^*$ is the point-wise conjugate of $Z$, and $C$ is some real constant.

Examples for synchronization and cancellation are seen in figure \ref{fig:comp_conv}. The behavior of the point-wise multiplication is similar to the accumulation and noise cancellation used in \cite{WermanComplexHough} to improve the Hough transform, as discussed in section \ref{synchronization}. The global phase factor $C$ does not affect the output's magnitude. The fact that different kernels yield the same magnitude with different angles, introduces some ambiguity in the model, which we refer to as \textit{phase ambiguity}. In chapter \ref{emp} we further address this issue.

With this interpretation, the complex convolution's output can be seen as a heat map where each pixel measures the similarity between the conjugate kernel's and the input patch's phase structure. Combining this notion with the results from the previous section, we conjecture that complex CNNs can serve as a regularized hypothesis class for problems with informative phase structure.

\begin{figure}[h!]
		\centering

		\begin{subfigure}[Patch = $e^{0.375\pi\i}\cdot$kernel$^*$ ]
				{\includegraphics[width=0.3\linewidth]{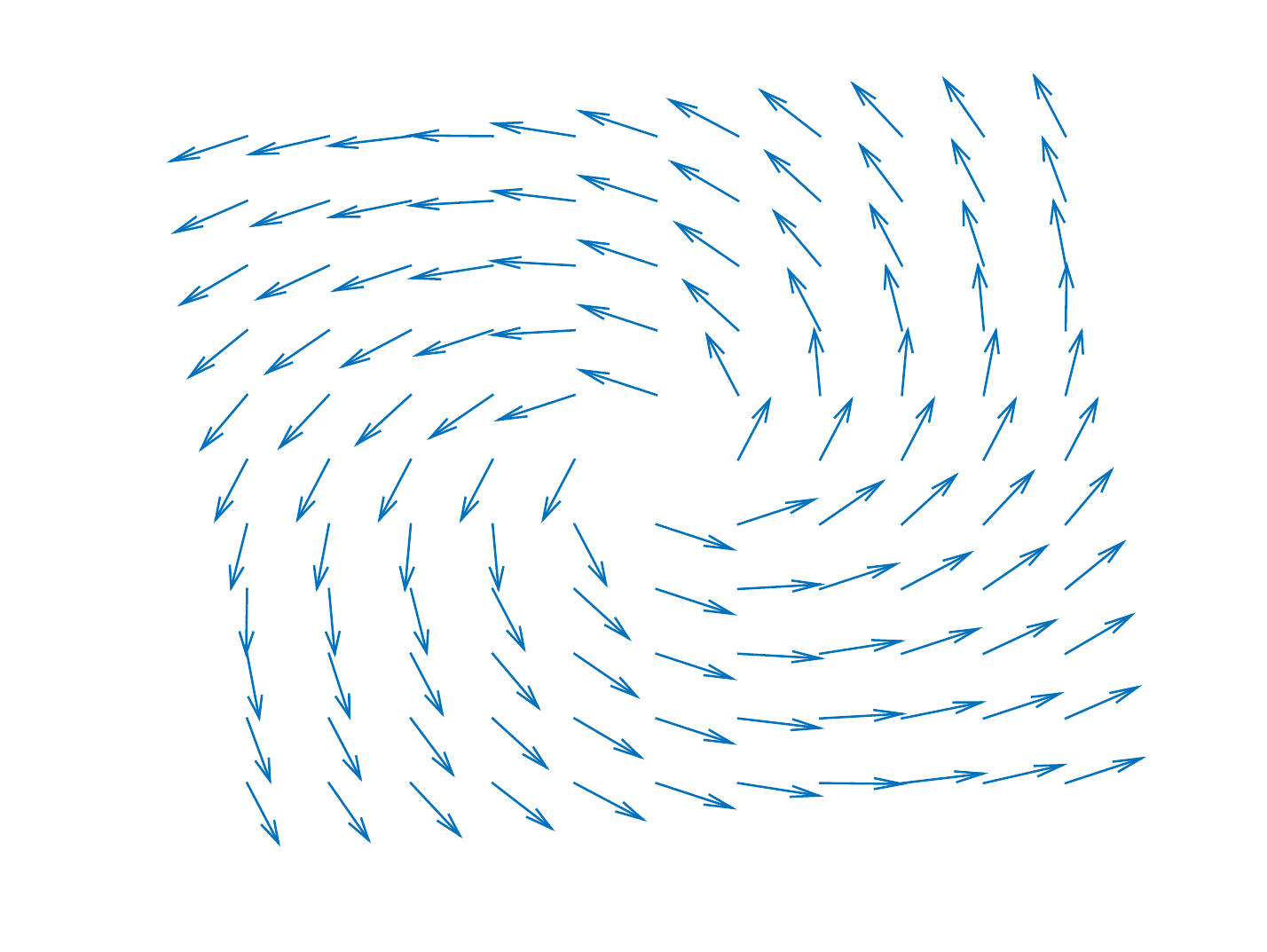}\label{fig:comp_conv_b}}				
		\end{subfigure} 
		\;
		\begin{subfigure}[Conjugate kernel]
			{\includegraphics[width=0.3\linewidth]{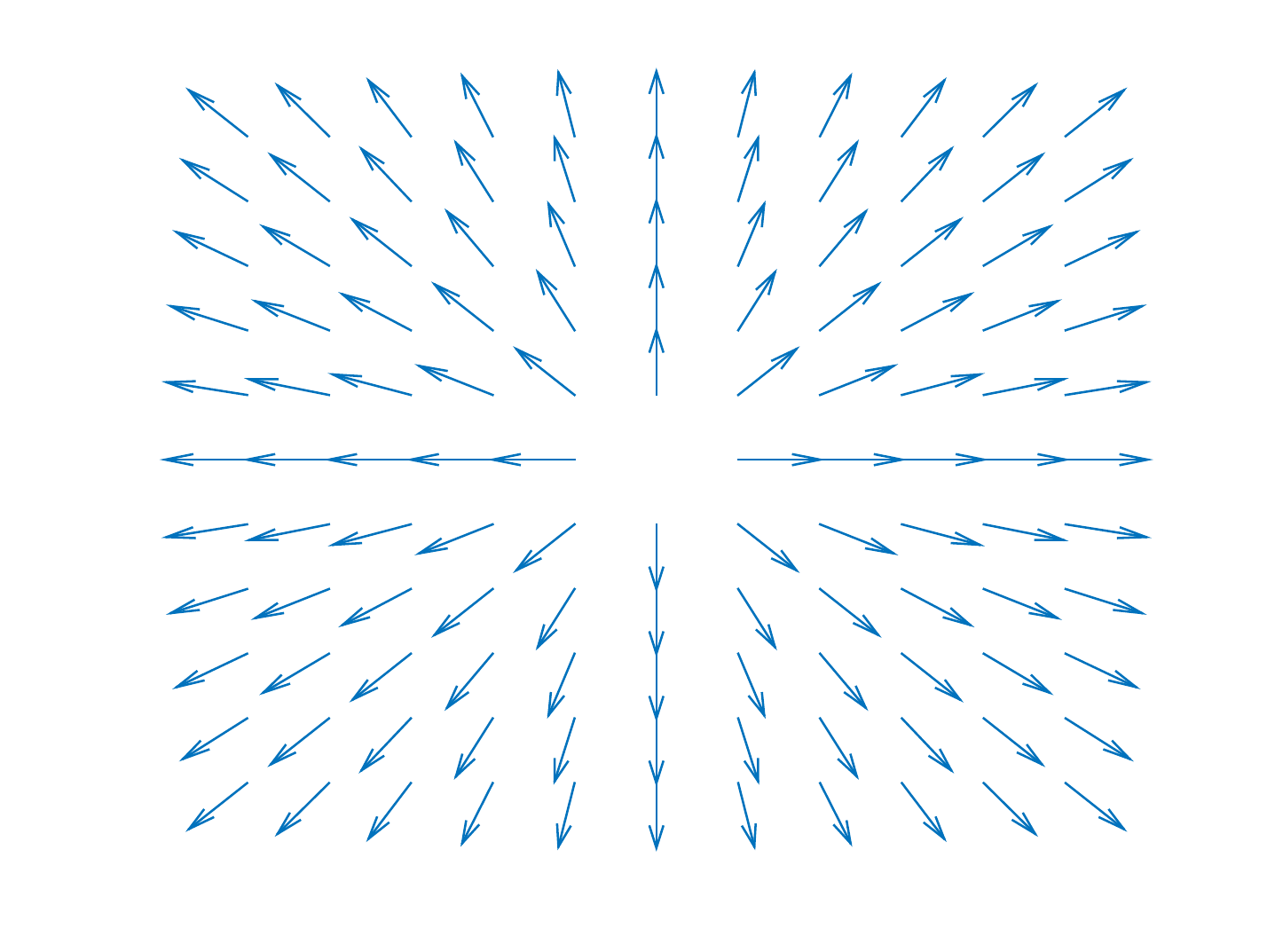}\label{fig:comp_conv_a}}				
		\end{subfigure}
		\;
		\begin{subfigure}[Point wise multiplication]
				{\includegraphics[width=0.3\linewidth]{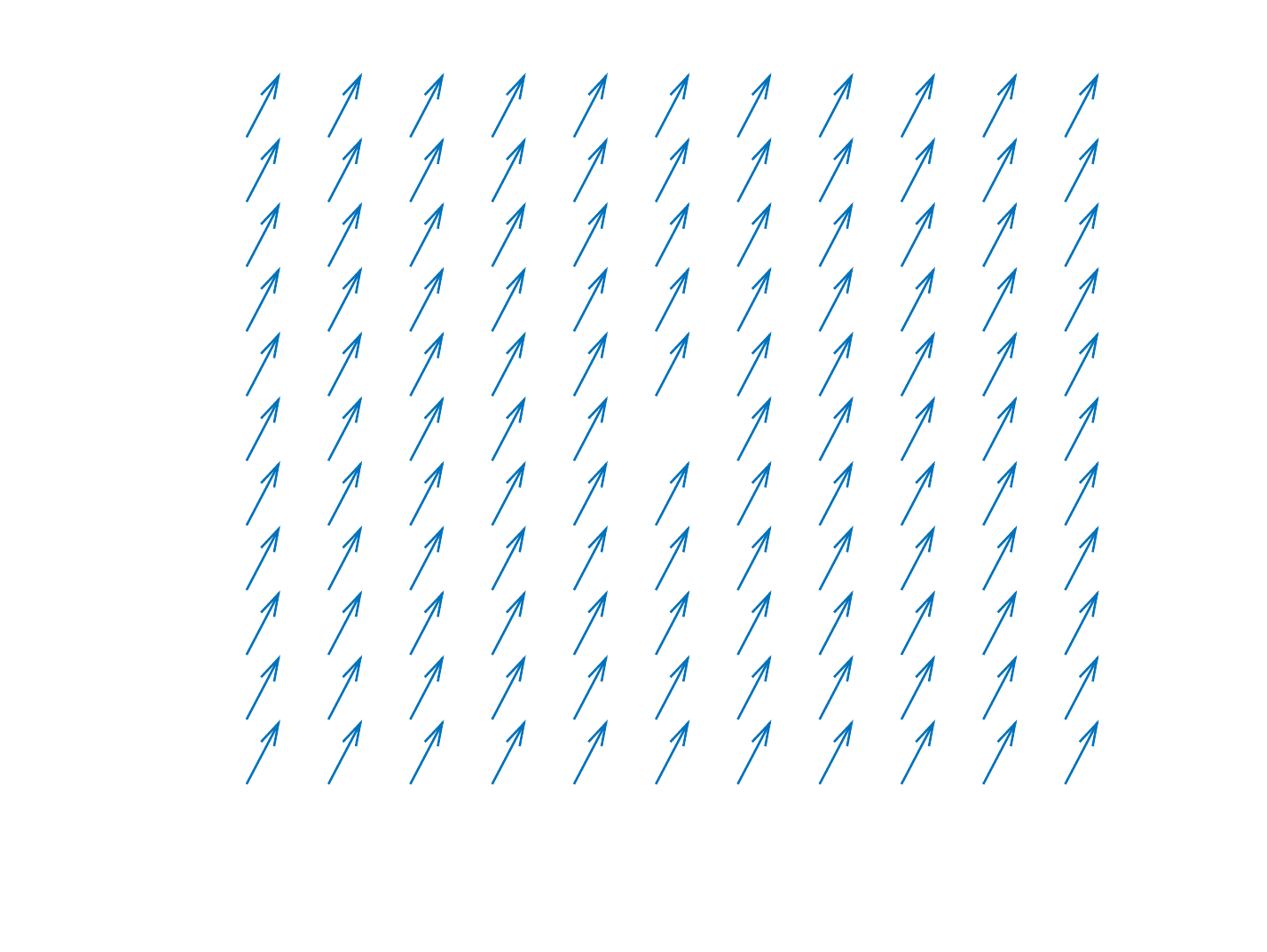}\label{fig:comp_conv_c}}				
		\end{subfigure}

		\begin{subfigure}[Random patch]
				{\includegraphics[width=0.3\linewidth]{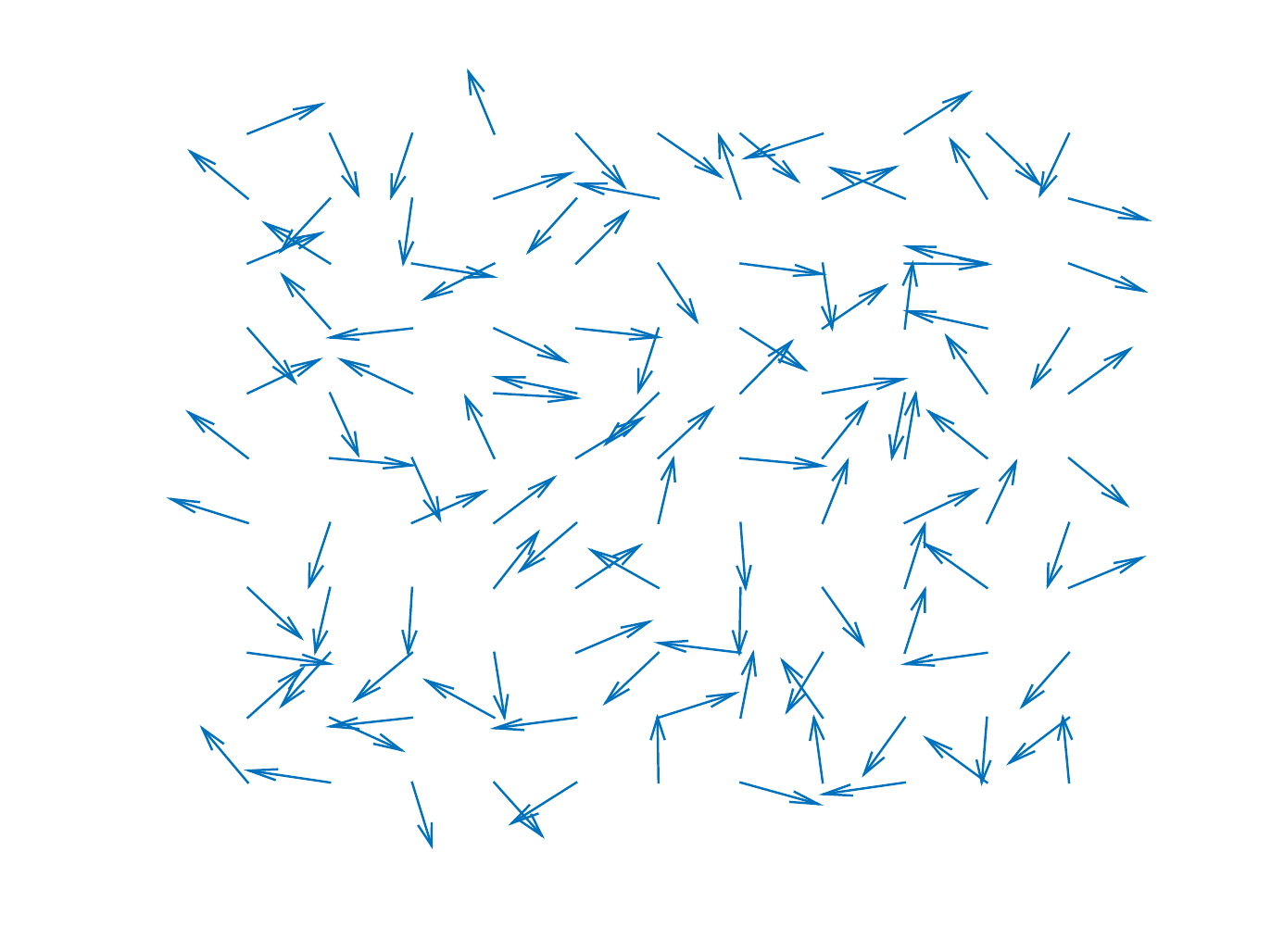}\label{fig:comp_conv_e}}				
		\end{subfigure} 
		\;
		\begin{subfigure}[Conjugate kernel]
			{\includegraphics[width=0.3\linewidth]{kernel_conj}\label{fig:comp_conv_d}}				
		\end{subfigure}
		\;
		\begin{subfigure}[Point wise multiplication]
				{\includegraphics[width=0.3\linewidth]{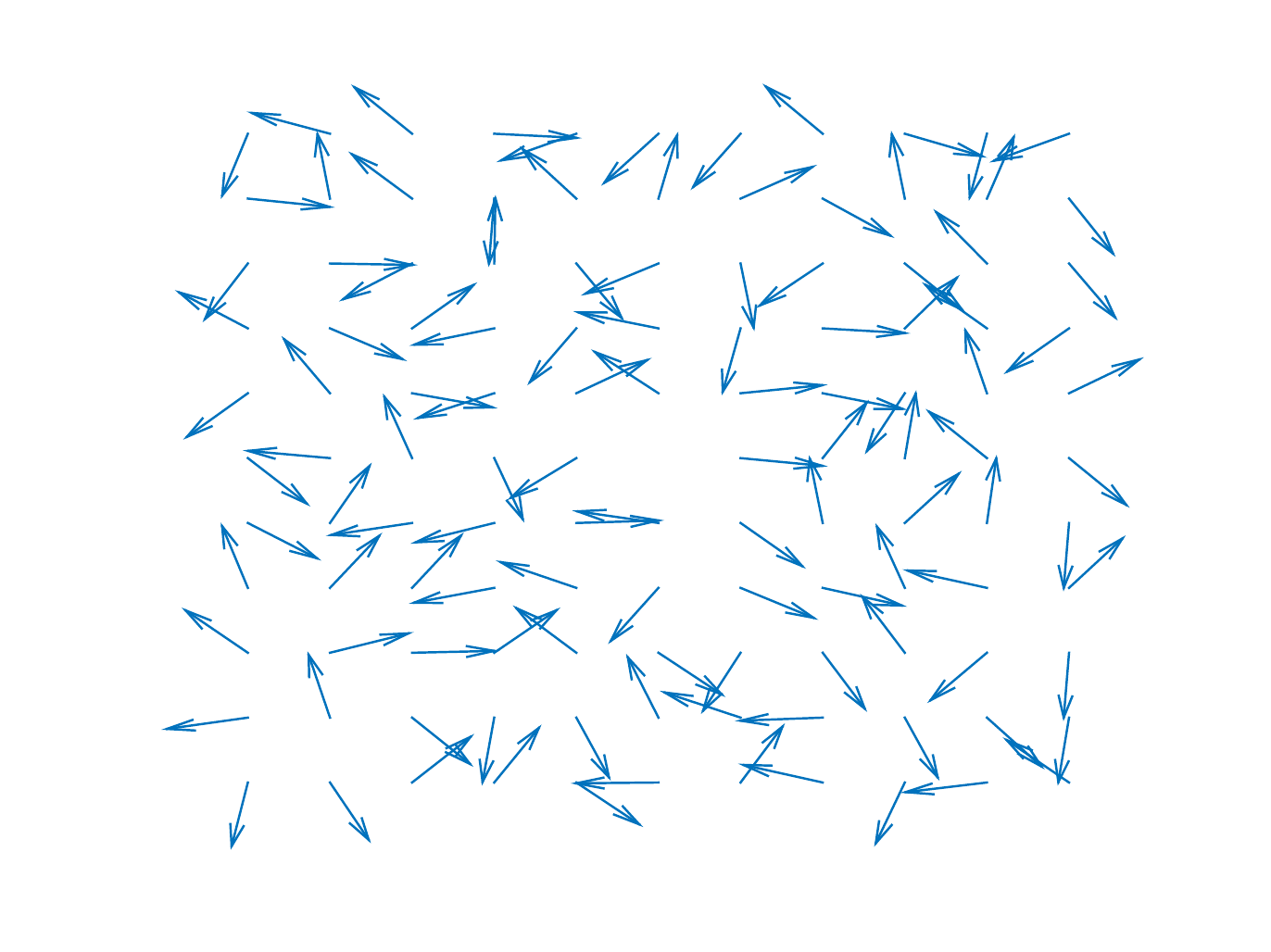}\label{fig:comp_conv_f}}				
		\end{subfigure}
		\caption{Examples of the synchronization effects of the point-wise multiplication. In the upper row the input patch \subref{fig:comp_conv_b} and conjugate kernel \subref{fig:comp_conv_a} share the same phase structure up to a point-wise multiplication by $e^{\i \frac{3\pi}{8}}$. The point-wise multiplication result in \subref{fig:comp_conv_c} is synchronized, all values have the same phase. In the bottom row the input patch \subref{fig:comp_conv_e} has no meaningful phase structure, and so does the multiplication in \subref{fig:comp_conv_f}. Both patches has mean magnitude of $1$, but the sum of the values in \subref{fig:comp_conv_c} has a magnitude over $20$ times larger then the sum of \subref{fig:comp_conv_f}.
\label{fig:comp_conv}}
\end{figure}

This characterization implies that the input of complex CNNs should be a complex valued representation with a meaningful phase structure. In the case of images, possible complex representations include the Fourier representation, wavelets, gradients, and optical flow. The Fourier representation does not preserve the locality properties of images, and therefore does not suit CNNs. Gradients and optical flow are usually represented as a two dimensional vectors, which are equivalent to complex numbers. There are many other possibilities, and each representation should be chosen specifically for the task at hand.

\clearpage
\section{Empirical Study - Cell Identification} \label{emp}

In this chapter we evaluate the complex CNN model by considering the problem of cell identification. Cells are circular shaped, and as such have a typical gradient image with a prominent phase structure. Complex valued CNNs might use this structure to produce good results, in a similar manner to the one discussed in section \ref{comp_conv}. We focus on evaluating the complex CNN model, and not on solving this specific problem. To that end, we use a minimalistic network and perform no major manipulations of the data.

We construct a complex CNN for the task of determining whether a given image patch contains a cell, and compare this network with its real valued counterpart. The two networks show comparable results, although the complex network suffers from convergence difficulties. To check the claim that complex CNNs act as a regularization, we examine the behavior of the loss as the optimization progresses. The real valued CNN is shown to be significantly more vulnerable to overfitting. To see whether the network utilizes the phase structure, we visualize the first convolution's kernels. Finally, we comment on the numerical difficulties encountered when training the complex network.

\subsection{Experimental Details}

\begin{wrapfigure}{R}{0.3\textwidth}
	\includegraphics[width=0.3\textwidth]{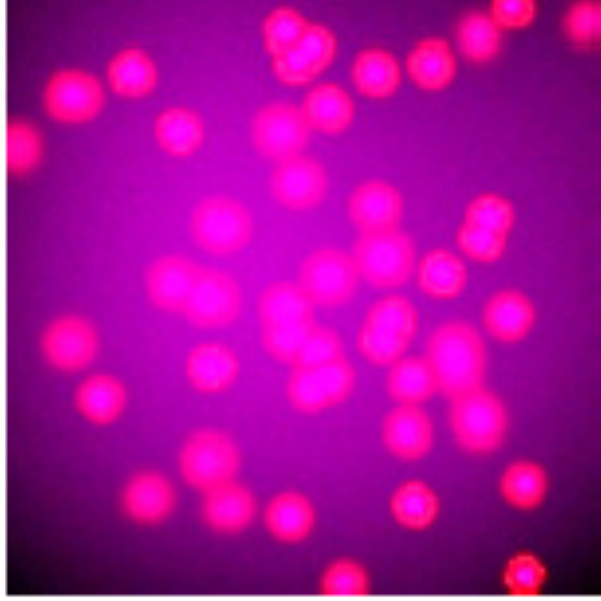}
	\caption{Simulated fluorescence microscopy image created by SIMCEP \label{fig:cell_img}.}
\end{wrapfigure}

In our experiment we use simulated fluorescence microscopy images, taken from  \cite{Lehmussola2007}. These are color images of multiple cells, an seen in figure \ref{fig:cell_img}. To create our dataset we simulate $150\times150$ color images, transform them to gray-scale, and compute their derivatives using the Sobel kernel.   Each gradient image is cropped to a $100$ non overlapping $15 \times 15$ patches.  The real network's input is the derivatives corresponding to a patch, $I_x,I_y$, and the complex network's input is $I_x+\i I_y$. The label assigned to each patch is "cell" if it has at least $10$ pixels belonging to a cell. Example gradients and labels are shown in figure \ref{fig:patches}. The patches were linearly normalized to have values between  $0$ and $1$. Both the training and test sets consist of $10,000$ patches taken from $100$ images. 

\begin{figure}[h!]
	\centering
	\includegraphics[width=0.6\textwidth]{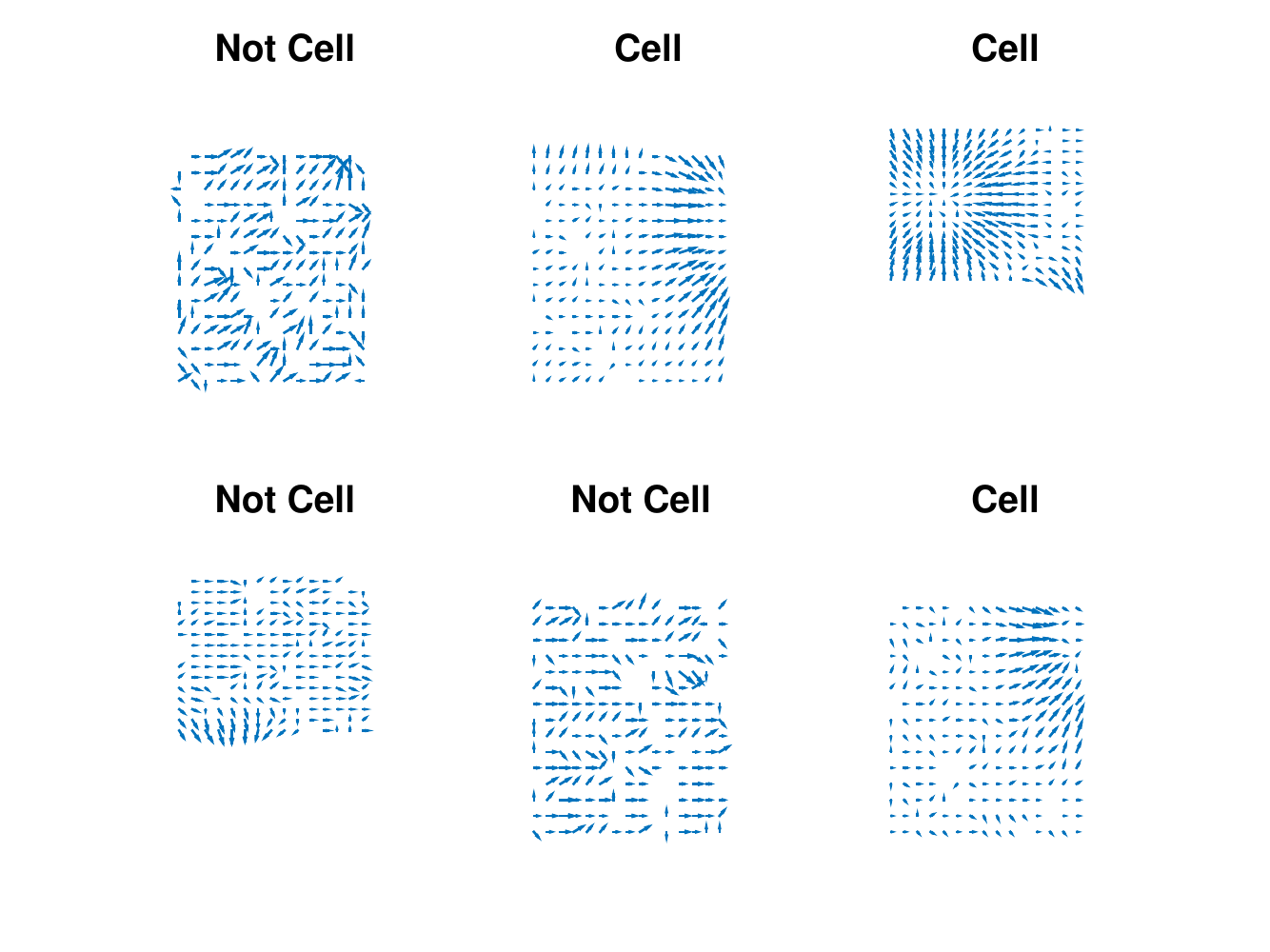}
	\caption{Examples of the networks' input - patches' gradients. The gradients are treated as a complex valued, and shown as a vector field. The labels are shown above each patch.\label{fig:patches}}
\end{figure}

The complex network's architecture we use consists of two convolution layers, each followed by an activation function layer, and a pooling layer. The kernels' sizes in both convolution layers are $5\times 5$ pixels, as the radius of cells is of the order of $5$ pixels. The first pooling layer has a window size of $2\times2$ with a stride of $1$, and the second performs global pooling\footnote{By global pooling we mean pooling over the entire spatial dimensions, across channels, as in \cite{Maxout}.}. As the labels are real valued, we add a projection layer. The resulting network is illustrated in figure \ref{fig:architecture}.

\begin{figure}[h!]
	\centering
	\includegraphics[width =1\textwidth]{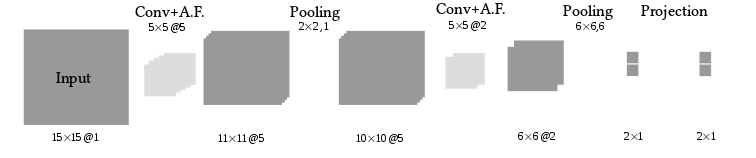}
	\caption{The complex network architecture, with two convolutions, activation function and pooling. To obtain real valued labels, a projection layer is  added. Best results achieved with $\relu$ activation function, $\abs{\cdot}$-pooling and $\abs{\cdot}^2$ projection.\label{fig:architecture}}
\end{figure}

As discussed in chapter \ref{construction}, there are several non trivial building blocks in the complex network, each having multiple options. These include the activation function, pooling method and projection layer. In this network we use $\relu$ as the activation function, and max pooling by magnitude. Other activation functions and pooling methods yielded comparable or inferior results. Since the last layers before the projection are $\relu$ and max pooling, many of the projection layer's inputs are $0$. The $\abs{\cdot}$ function is non differentiable at $0$, so such a setting is problematic for the optimization process, as described in chapter \ref{comp_backprop}. To overcome this, the $\abs{\cdot}^2$ projection was used instead.

We compare the complex network to its real valued equivalent, in the sense described in chapter \ref{regularization}. This network shares the same architecture as the complex one, only with twice the channels and convolution kernels. By construction, the last layer of the real network consists of twice as many neurons as the complex one. As the labels are binary, the final layer has to be two channeled, so we add a fully connected (affine) layer replacing the projection layer in the complex network. The resulting network is shown in figure \ref{fig:real_arch}.

\begin{figure}[h!]
	\centering
	\includegraphics[width = 1\textwidth]{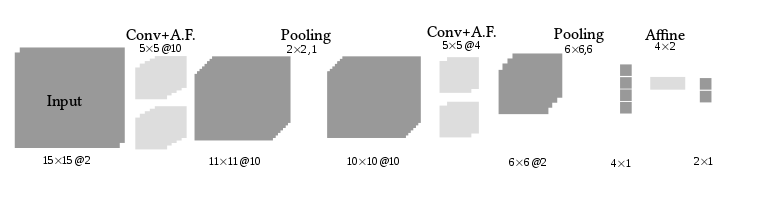}
	\caption{The real network architecture equivalent to the complex one in \ref{fig:architecture}. There are twice as many channels and kernels. To obtain an output of two classes, rather then four the projection layer from \ref{fig:architecture} was replaced by an affine layer. \label{fig:real_arch}}
\end{figure}

Both networks were trained by minimizing the multi-class logistic loss using SGD with Nesterov's acceleration, as presented in \cite{Nesterov}. As we aim to check the regularization capabilities of the model, no regularization methods were applied. For the same reason, the momentum coefficient and learning rate were chosen to maximize the performance over the training set. For the complex model the momentum coefficient is $0.9$ and the learning rate is $0.01$ for the first $2,000$ iterations, and $0.001$ afterwards. For the real valued network, the momentum coefficient is $0.9$ and the learning rate is fixed at $0.1$. We trained both models for $20,000$ iterations with a batch size of $100$, and used the initialization scheme suggested in \cite{Glorot2010}.

\subsection{Comparison With a Real Network}\label{comparison}

We consider the complex model and its real counterpart after training each to achieve the minimal training loss, without any regularization. The final losses and accuracies are presented in table \ref{table:results}. Overall the accuracies are comparable, with the real model performing slightly better. 

\begin{table}[h!]
	\centering
	\begin{tabular}{|l|c|c|c|c|}
		\hline
		& Train loss & Train Accuracy & Test loss & Test Accuracy  \\
		\hline
		Complex network & 0.056 & 97.4\% & 0.0690 & 97.3\% \\
		\hline
		Real network & 0.007 & 99.8 \%  & 0.1450 &  97.5\% \\
		\hline
		
	\end{tabular}
	\caption{ The results of the real and complex model over the training and test set\label{table:results}. The real model's test loss is significantly higher than its training loss, which suggests overfitting. In the complex model both losses are close, which advocates to regularization capabilities. The accuracies do not follow this pattern.}
\end{table}

The training loss of the real network is much lower then its test loss, while those of the complex network are comparable.  Figure \ref{fig:regularization_comparison_graphs} shows the loss rates as the training progresses. In the real network, after a quick decrease of both losses, the training loss nearly vanishes and the test loss rises. Clearly, the real model suffers from overfitting. On the other hand, the complex network does not present overfitting, as the training and test loss of the complex network remain close, and lie between the real network's train and test loss. These results suggest that the complex model serves as a regularization.

The accuracies, however, do not present the same pattern. The real network's test accuracy does not decrease as the loss rises, and is higher than that of the complex network. While these results are puzzling, the possible regularization capabilities of the complex network are not undermined, as they can only be measured with respect to the loss being minimized. More data is needed to see if this phenomenon is repeated for different tasks, and network architectures. 

\begin{figure}[h!]
	\centering
	\begin{subfigure}[Complex network convergence]
		{\includegraphics[width=0.9\textwidth]{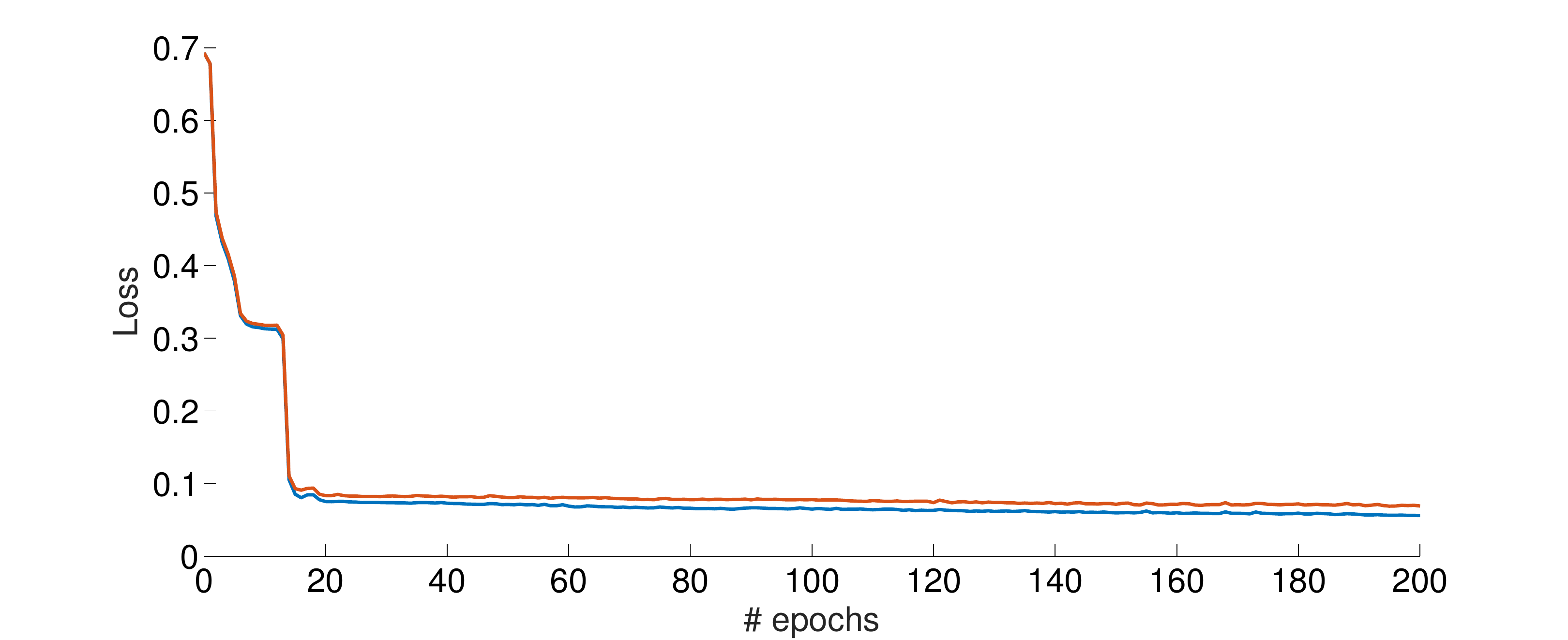}}
	\end{subfigure}
	
	\begin{subfigure}[Real network convergence ]
		{\includegraphics[width=0.9\textwidth]{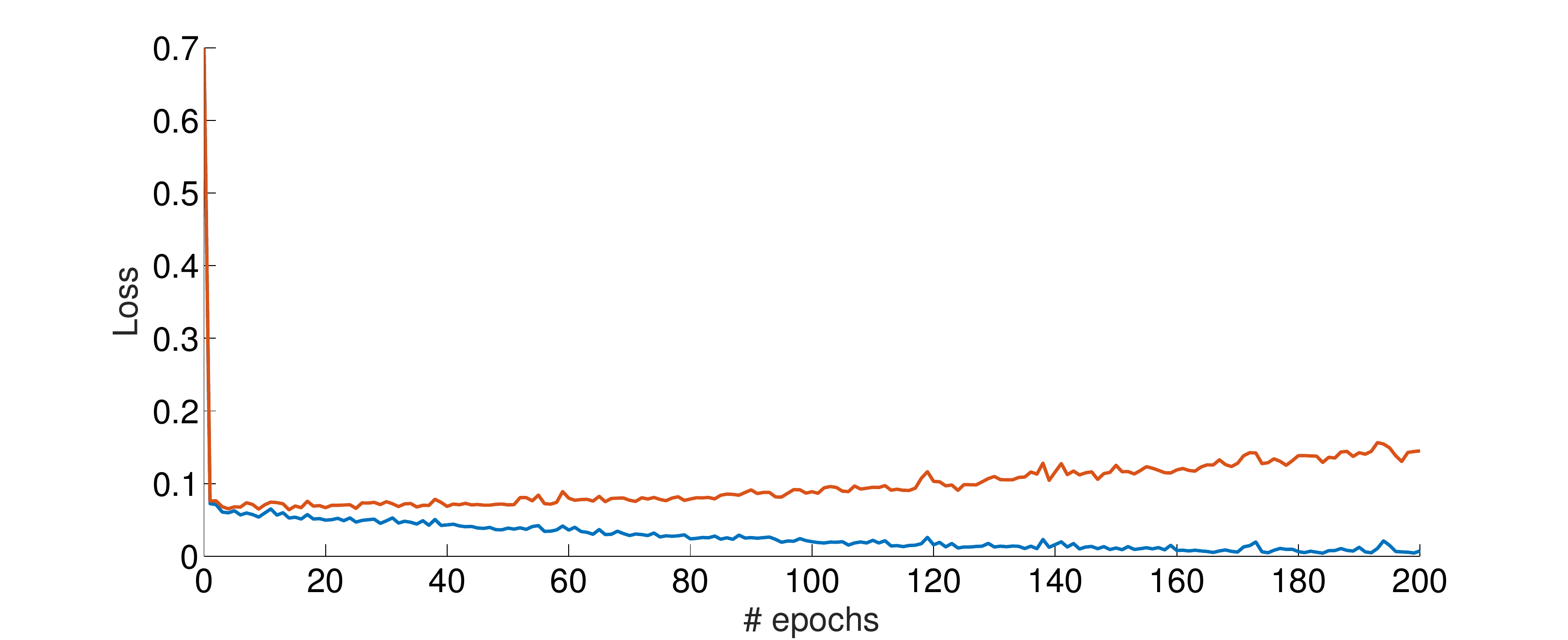}}
	\end{subfigure}

	\caption{The convergence of the real and complex networks with the algorithm's progress. An epoch is the number of iterations in which the total number of examples chosen is equal to the size of the training set, in our setting one epoch is a 100 iterations. In the blue line, the training loss, and in the red line the test loss. The real model suffers from overfitting, while the complex one does not.\label{fig:regularization_comparison_graphs}}
\end{figure}

\subsection{Numerical Difficulties}\label{convergence}

The training of the complex network proved difficult. To demonstrate this effect, we trained the network $20$ times with the same parameters stated above, for $10,000$ iterations. The only differences between trials are due to the random parts of the algorithm - the initialization and mini batch choice in each SGD iteration. Only $4$ times of $20$ has the network achieved training loss that is close to its best. The loss rate over the training set across the training process for these $20$ trials is plotted in figure \ref{fig:rep_lines}. This plot demonstrates the sensitivity of the network to the randomization effects, and its great instability.

\begin{figure}[h!]
	\centering
	\includegraphics[width=0.9\textwidth]{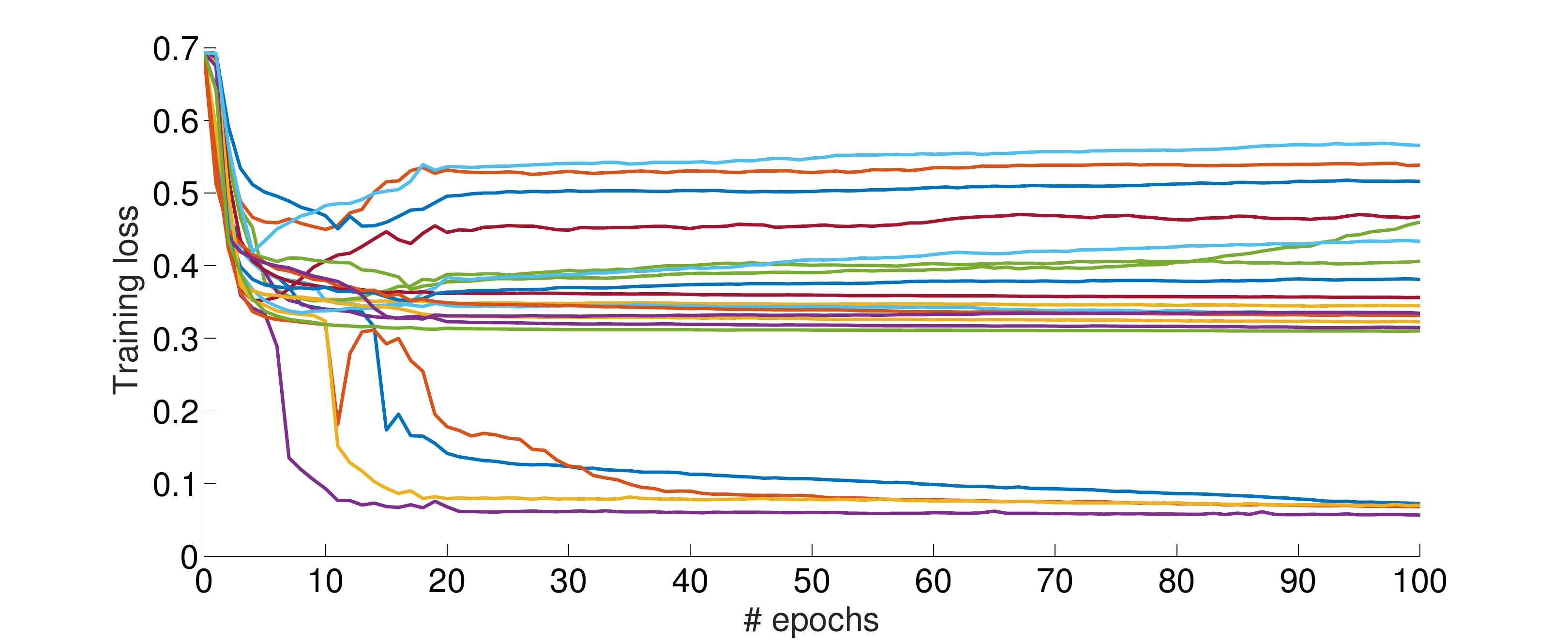}
	\caption{Repeated training of the complex network. Each line is the training loss across the optimization epochs of a single trial. The learning rate is reduced after 20 epochs, for optimal convergence.
		The training process is unstable, and sensitive to the randomization effects. Not all the trials have converged, and among the ones that did, most did not achieve the global minimum. \label{fig:rep_lines}}
\end{figure}

In a similar experiment with the real valued network, all trials yielded similar results, hence the difficulties are likely due to the complex nature of the network. Previous works concerning complex ANNs, reported numerical difficulties as well, for example in \cite{Kim2002}. Unfortunately, they do not shed light on the sources of these difficulties or ways to overcome them.

\subsection{Qualitative Analysis of Kernels} \label{visualization}

Having established that the complex CNN indeed operates as a regularization method, we turn to analyzing the resulting complex model. It is a common practice in CNNs to visualize the first convolution's kernels, to obtain some intuition regarding the network's mechanism. In this section we visualize the kernels of the complex network, and examine if indeed the network identifies common phase structures. This visualization also helps resolving the phase ambiguity discussed in \ref{comp_conv}.

In section \ref{comp_conv}, it was shown that a complex convolution measures the similarity between the input and the kernel's conjugate. It has also been established that two kernels that differ only by a global phase factor are equivalent in their influence. In figure \ref{fig:comp_kernels} the conjugates of the first convolution layer's kernels are presented, with the mean magnitude above each kernel.

The upper left kernel has a much higher mean magnitude than the rest, which suggests it is important to the network's operation. This kernel also has a very distinct phase structure, which resembles that of a cell's center, up to a global multiplicative phase factor. Indeed, if we multiply the upper left kernel by $e^{\frac{\i\pi}{3}}$, we obtain a remarkably similar phase structure to a cell's center. For the sake of clarity we will refer to this kernel as the cell kernel. In figure \ref{fig:cell_kernel_comparizon} we show the cell kernel, with and without the global phase factor, and compare it with a cell's center from an example patch.

Repeated trainings of the network all yielded a similar kernel, which raises the question what is special about the global phase $e^{\frac{\pi\i}{3}}$.  We suggest that this is the phase that allows the response to have positive real and imaginary parts. This is crucial, since otherwise the response would be zeroed out by the following $\relu$ operation. In figure \ref{fig:response} we present the result of convolving the patch in \ref{bla} with the discussed kernel. Indeed the response contains mainly vectors with positive real and imaginary parts.

\begin{figure}[h!]
	\centering{
		\includegraphics[width=0.8\linewidth]{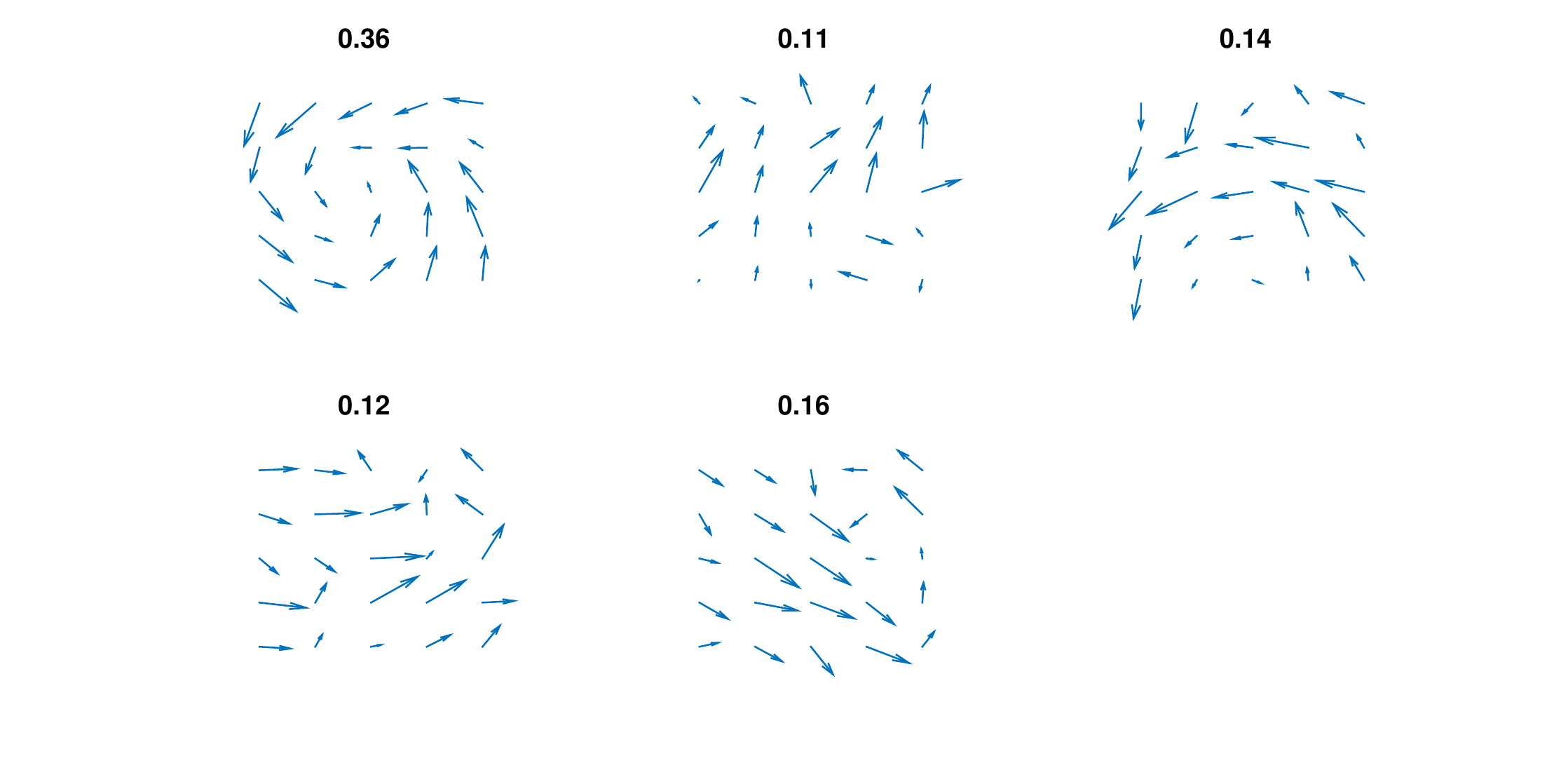}}
	\caption{Kernels of the first convolution of the trained complex network. The kernels are scaled to the same mean for the sake of the presentation. The title of each kernel is it's original mean magnitude. The upper left kernel, referred to as the cell kernel, has a significantly higher mean absolute value, and a prominent phase structure.\label{fig:comp_kernels}}
\end{figure}

\begin{figure}[h!]
	\centering
		\begin{subfigure}[The cell kernel]
			{\includegraphics[width=0.2\textwidth]{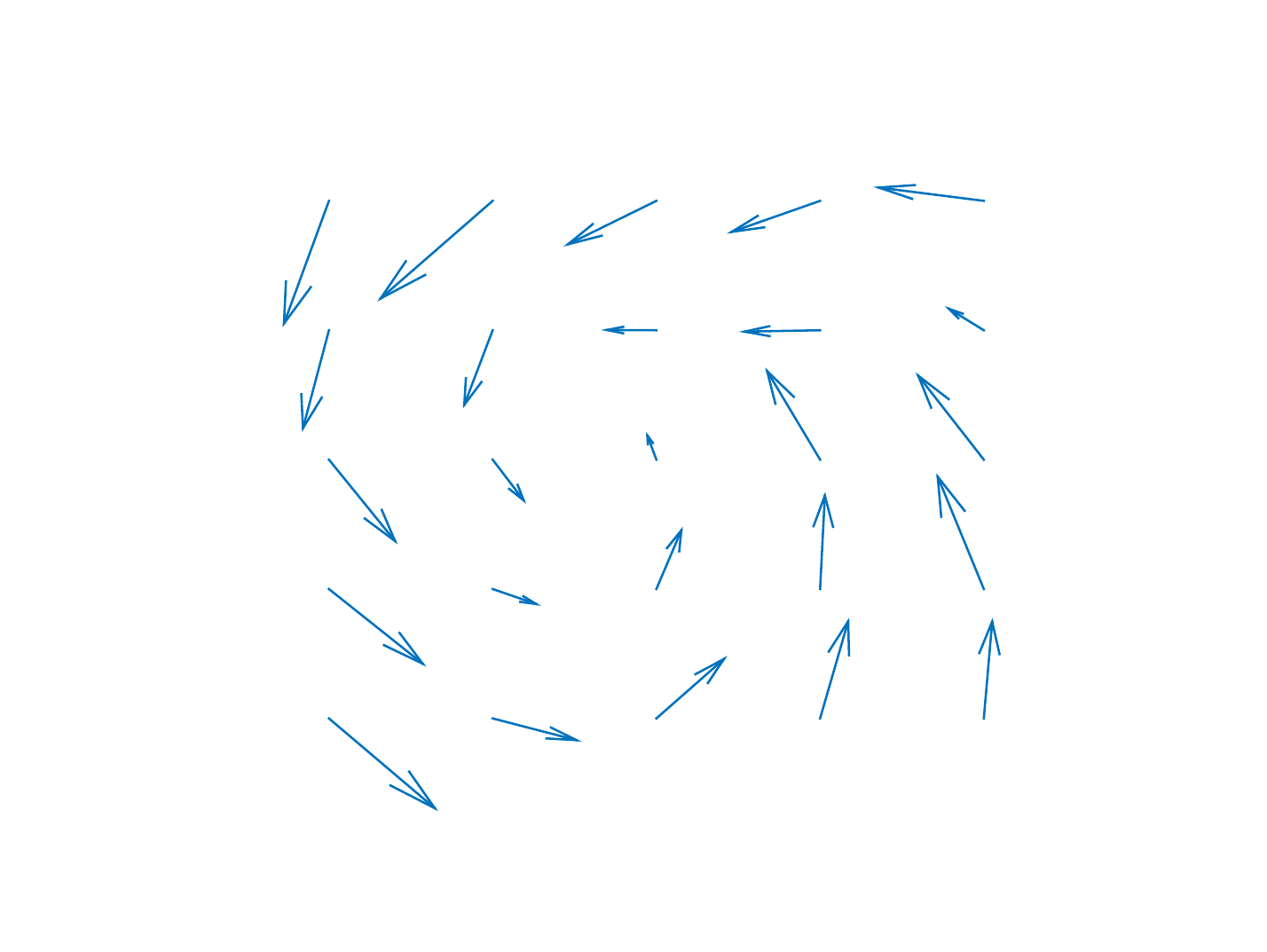} \label{fig:cell_kernel}}
		\end{subfigure}
		\quad
	\begin{subfigure}[The cell kernel multiplied by $e^{\frac{\pi\i}{3}}$]
		{\includegraphics[width=0.2\textwidth]{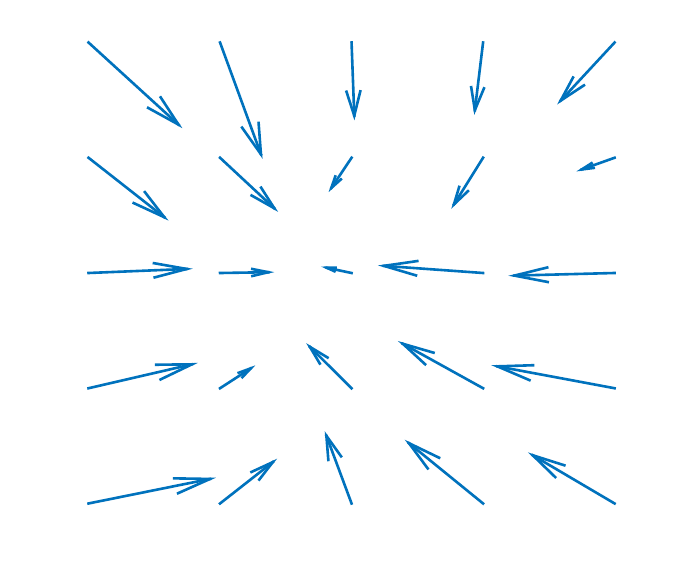} \label{fig:multiplied_cell_kernel}}
	\end{subfigure}
	\quad
	\begin{subfigure}[A close up of the rectangle in \ref{bla}]
		{\includegraphics[width=0.2\textwidth]{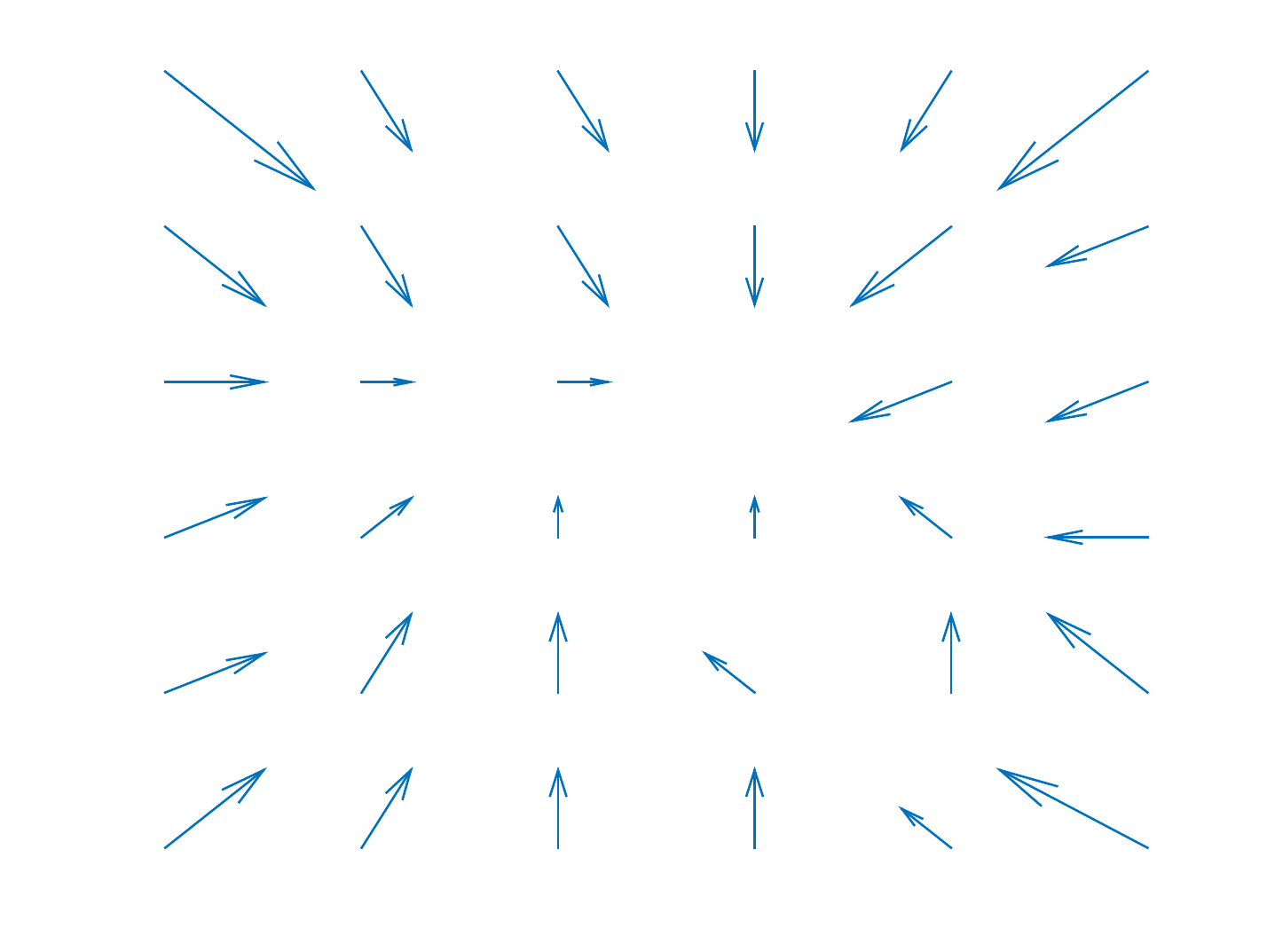}}
	\end{subfigure}
	\quad
	\begin{subfigure}[Gradients' image of a cell]
		{\includegraphics[width=0.2\textwidth]{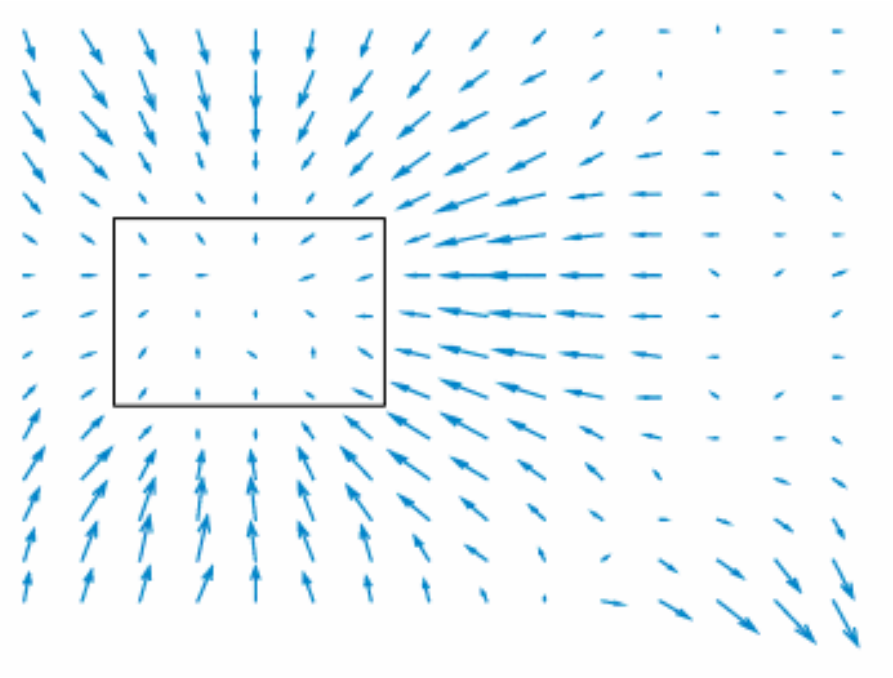}\label{bla}}
	\end{subfigure}

	\caption{Comparison between the learned kernel and a cell center. In the left, the learned kernel multiplied by a global phase. In the right, an example of a cell's gradient image. A close up of the black rectangle in this image is presented in the middle section.\label{fig:cell_kernel_comparizon}}
\end{figure}

\begin{figure}[h!]
	\centering
	\includegraphics[width=0.3\textwidth]{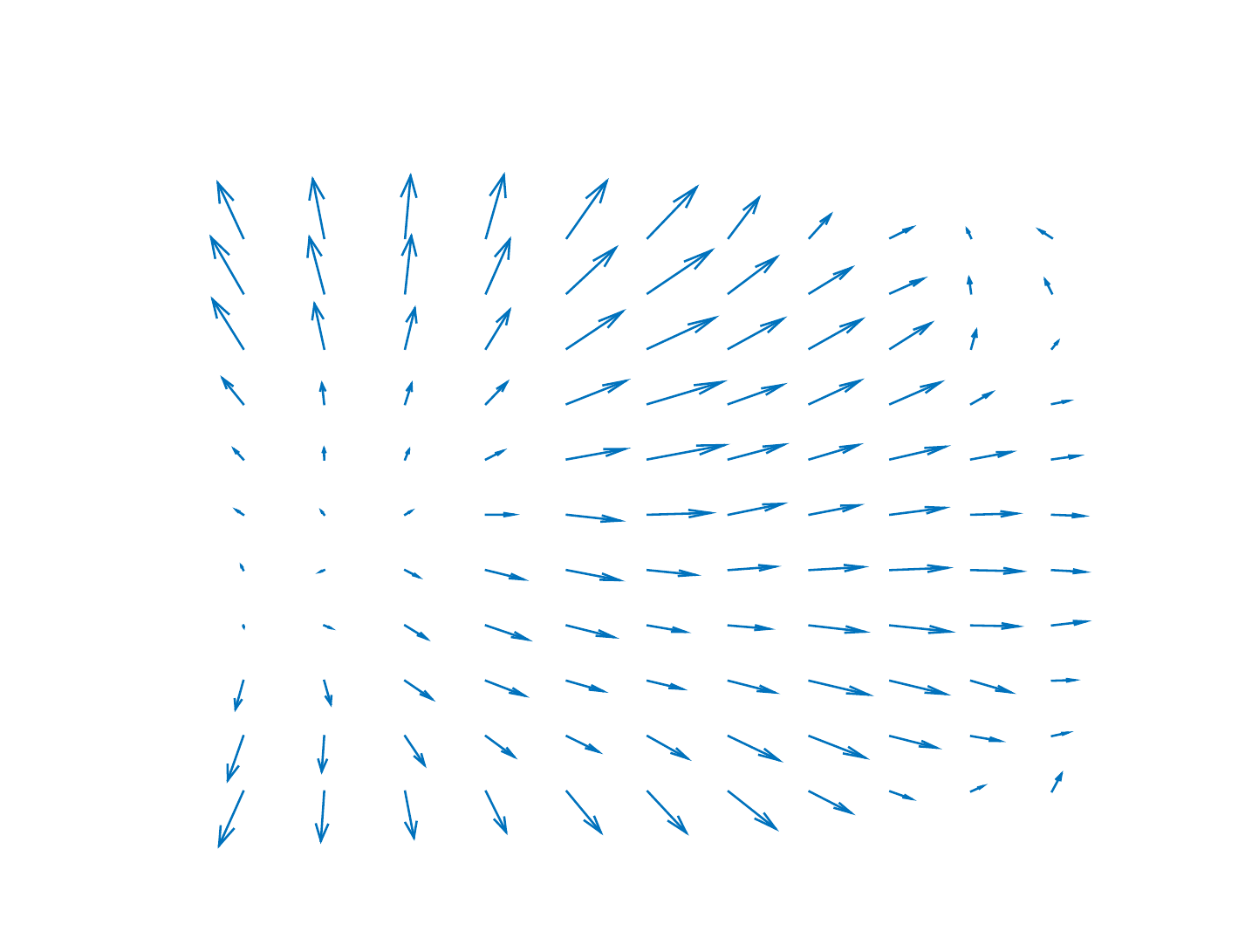}
	\caption{Convolution of the cell patch \ref{bla} with the cell kernel.\label{fig:response}}
\end{figure}

\clearpage
\section{Conclusion and Future Work}

In this work we presented a complex valued CNN model, built as a generalization of the real model, with complex input and weights. Linear operations generalize trivially to the complex domain, while comparison based operations, such as $\relu$ and max pooling, are ill-defined over complex inputs due to the lack of order in the complex field. We described the problems encountered along with possible solutions. We also handled the optimization method for this network, and modified the well known back propagation algorithm.

A theoretical analysis reveals that the resulting model is a regularized subclass of CNNs. A complex convolution is a spacial case of a real valued convolution with twice the parameters, and a tight constraint over the weights. This constraint creates a model cut out for detecting meaningful phase structure.

We explored this model in an empirical setting, by considering the binary classification problem of cell detection - given an image patch, decide whether it contains a cell or not. The input data was gradients images of circular cells, that have a revealing phase structure. 

We trained a complex network and its real valued counterpart for this classification task. The training process of the complex network was riddled with difficulties. Given the best learning parameters, only $20\%$ of the trials converged to a non local minima. However, in the trials that did converge, the results were promising. There was no overfitting present in the complex network, while the real network suffered from it considerably. Moreover, inspecting the kernels of the first convolution layer of the complex network, we have shown that it detected the phase structure typical for a cell center.

Further work should address the optimization difficulties in the training process of the complex model, as this seems to be a major stumbling block for successful application of the model. Given a satisfactory training method, complex networks should be used for other, possibly more complicated challenges. Tackling additional tasks would gain us better understanding of the importance of phase structure in different problems, and hence the benefits of the regularization capabilities. Further experiments should explore the different possibilities suggested for the model's construction, such as pooling by softmax. 

We should also explore the merits of the complex model using different inputs. These include additional two-dimensional image representations, such as optical flow. The model could also benefit other natural signals with an innate complex representation, such as voice signals.

\newpage
\bibliography{thesis}
\bibliographystyle{plain}

\end{document}